\documentclass[twoside,11pt]{article}
\pdfoutput=1
\usepackage{jmlr2e}
\usepackage{helvet}  
\usepackage{courier}  
\usepackage{url}  
\usepackage{graphicx}  
\usepackage{comment}
\usepackage{subfigure}
\usepackage{enumitem} 
\usepackage{amsthm}
\usepackage{amsmath}
\usepackage{amssymb}
\usepackage{makecell}
\usepackage{bm}
\usepackage{algorithm}
\usepackage{algorithmic}
\usepackage{multirow}

\newtheorem{prop}{Proposition}
\newtheorem{mydef}{Definition}
\newtheorem*{myplm}{Maximum Agreement Problem}
\newtheorem*{myplm2}{Decision Problem of MAP}
\newtheorem{thm}{Theorem}
\newtheorem{lemma}{Lemma}

\jmlrheading{}
\ShortHeadings{$\tau$-FPL: Tolerance-Constrained Learning in Linear Time}{Zhang, Li, Pu, Wang, Yan and Zha}
\firstpageno{1}

\begin{document}

\title{$\tau$-FPL: Tolerance-Constrained Learning in Linear Time \thanks{A preliminary version of this work appeared in the Proceedings of the Thirty-Second AAAI Conference on Artificial Intelligence, 2018.}}

\author{\name Ao Zhang \email az.aozhang@gmail.com \\
       \addr School of Computer Science and Software Engineering\\
       East China Normal University, Shanghai, China
       \AND
       \name Nan Li \email nanli.ln@alibaba-inc.com \\
       \addr Institute of Data Science and Technologies\\
       Alibaba Group, Hangzhou, China
       \AND 
       \name Jian Pu \email jianpu@sei.ecnu.edu.cn\\
       \addr School of Computer Science and Software Engineering\\
       East China Normal University, Shanghai, China
       \AND 
       \name Jun Wang\email jwang@sei.ecnu.edu.cn\\
       \addr School of Computer Science and Software Engineering\\
       East China Normal University, Shanghai, China
       \AND 
	   \name Junchi Yan \email yanesta13@163.com\\
	   \addr IBM Research -- China, Shanghai, China\\
	   School of Computer Science and Software Engineering\\
	   East China Normal University, Shanghai, China
		\AND 
		\name  Hongyuan Zha \email zha@cc.gatech.edu\\
		\addr 
		College of Computing, Georgia Institute of Technology, Atlanta, USA\\
        School of Computer Science and Software Engineering\\
       East China Normal University, Shanghai, China
	}


\maketitle

\begin{abstract}
Learning a classifier with control on the false-positive rate plays a critical role in many machine learning applications. Existing approaches either introduce prior knowledge dependent label cost or tune parameters based on traditional classifiers, which lack consistency in methodology because they do not strictly adhere to the false-positive rate constraint. In this paper, we propose a novel scoring-thresholding approach, \emph{$\tau$-False Positive Learning} ($\tau$-FPL) to address this problem. We show the scoring problem which takes the false-positive rate tolerance into accounts can be efficiently solved in linear time, also an out-of-bootstrap thresholding method can transform the learned ranking function into a low false-positive classifier. Both theoretical analysis and experimental results show superior performance of the proposed $\tau$-FPL over existing approaches.
\end{abstract}


\section{Introduction}
In real-world applications, such as spam filtering \citep{drucker1999support} and medical diagnosing \citep{huang2010bayesian}, the loss of misclassifying a positive instance and negative instance can be rather different. For instance, in medical diagnosing, misdiagnosing a patient as healthy is more dangerous than misclassifying healthy person as sick. Meanwhile, in reality, it is often very difficult to define an accurate cost for these two kinds of errors \citep{liu2010learning,zhou2016large}. In such situations, it is more desirable to keep the classifier working under a small tolerance of false-positive rate (FPR) $\tau$, i.e., only allow the classifier to misclassify no larger than $\tau$ percent of negative instances. Traditional classifiers trained by maximizing classification accuracy or AUC are not suitable due to mismatched goal.

In the literature, classification under constrained false-positive rate is known as Neyman-Pearson (NP) Classification problem \citep{scott2005neyman,lehmann2006testing,rigollet2011neyman}, and existing approaches can be roughly grouped into several categories. One common approach is to use \emph{cost-sensitive learning}, which assigns different costs for different classes, and representatives include cost-sensitive SVM \citep{osuna1997support,davenport2006controlling,davenport2010tuning}, cost-interval SVM \citep{liu2010learning} and cost-sensitive boosting \citep{masnadi2007asymmetric,masnadi2011cost}. Though effective and efficient in handling different misclassification costs, it is usually difficult to find the appropriate misclassification cost for the specific FPR tolerance.
Another group of methods formulates this problem as a constrained optimization problem, which has the FPR tolerance as an explicit constraint \citep{mozer2002prodding,gasso2011batch,Mahdavi2013}. These methods often need to find the saddle point of the Lagrange function, leading to a time-consuming alternate optimization. Moreover, a surrogate loss is often used to simplify the optimization problem, possibly making the tolerance constraint not satisfied in practice.
The third line of research is scoring-thresholding methods, which train a scoring function first, then find a threshold to meet the target FPR tolerance \citep{drucker1999support}. In practice, the scoring function can be trained by either class conditional density estimation~\citep{tong2013plug} or bipartite ranking~\citep{narasimhan2013relationship}. However, computing density estimation itself is another difficult problem. Also most bipartite ranking methods are less scalable with  super-linear training complexity.
Additionally, there are methods paying special attention to the positive class. For example, asymmetric SVM \citep{wu2008asymmetric} maximizes the margin between negative samples and the core of positive samples, one-class SVM \citep{ben2001support}  finds the smallest ball to enclose positive samples. However, they do not incorporate the FPR tolerance into the learning procedure either.

In this paper, we address the tolerance constrained learning problem by proposing \emph{$\tau$-False Positive Learning} ($\tau$-FPL). Specifically, $\tau$-FPL is a  scoring-thresholding method. In the scoring stage, we explicitly learn a ranking function which optimizes the probability of ranking any positive instance above the centroid of the worst $\tau$ percent of negative instances. Whereafter, it is shown that, with the help of our newly proposed Euclidean projection algorithm, this ranking problem can be solved in linear time under the projected gradient framework. It is worth noting that the Euclidean projection problem is a generalization of a large family of projection problems, and our proposed linear-time algorithm based on bisection and divide-and-conquer is one to three orders faster than existing state-of-the-art methods.
In the thresholding stage, we devise an out-of-bootstrap thresholding method to transform aforementioned ranking function into a low false-positive classifier. This method is much less prone to overfitting compared to existing thresholding methods. Theoretical analysis and experimental results show that the proposed method achieves superior performance over existing approaches.

\section{From Constrained Optimization to Ranking}
In this section, we show that the FPR tolerance problem can be transformed into a ranking problem, and then formulate a convex ranking loss which is tighter than existing relaxation approaches.

Let $\mathcal{X} =\{\bm{x} \mid \bm{x} \in \mathbb{R}^d: ||\bm{x}|| \leq 1 \}$ be the instance space for some norm $||\cdot||$ and $\mathcal{Y} = \{-1, +1\}$ be the label set, $\mathcal{S} = \mathcal{S}_+ \cup \mathcal{S}_-$ be a set of training instances, where $\mathcal{S}_+ = \{ \bm{x}_i^+ \in \mathcal{X}\}_{i=1}^m$ and $\mathcal{S}_- = \{ \bm{x}_j^- \in \mathcal{X}\}_{j=1}^n$ contains $m$ and $n$ instances independently sampled from distributions $\mathbb{P}^+$ and $\mathbb{P}^-$, respectively. Let $0 \leq \tau \ll 1$ be the maximum tolerance of false-positive rate. 
Consider the following Neyman-Pearson classification problem, which aims at minimizing the false negative rate of classifier under the constraint of false positive rate:
\begin{equation} 
\begin{split}
\min_{f, b} ~ ~ \mathbb{P}_{\bm{x}^+ \sim \mathbb{P}^+}\left(f(\bm{x}^+\right) \leq b)~~~~\textrm{\small s.t.} ~ ~ \mathbb{P}_{\bm{x}^- \sim \mathbb{P}^-}\left(f(\bm{x}^-) > b\right) \leq \tau,
\end{split}
\end{equation} 
where $f:\mathcal{X} \to \mathbb{R}$ is a scoring function and $b \in \mathbb{R}$ is a threshold. With finite training instances, the corresponding \emph{empirical risk minimization} problem is
\begin{equation}\label{empirical plm}
\begin{split}
\min_{f, b} ~&~  {\cal L}_{emp}(f, b)  = \frac{1}{m} \sum \nolimits_{i=1}^m\mathbb{I}\left(f(\bm{x_i^+}) \leq b\right) \\\\
\textrm{\small s.t.} ~&~ \quad \frac{1}{n} \sum \nolimits_{j=1}^n\mathbb{I}\left(f(\bm{x_j^-}) > b\right) \leq \tau,
\end{split}
\end{equation} 
where $\mathbb{I}(u)$ is the indicator function.

The empirical version of the optimization problem is difficult to handle due to the presence of non-continuous constraint; we introduce an equivalent form below.
\begin{prop}\label{tcp-rp}
	Define $f(\bm{x}_{[j]}^-)$ be the $j$-th largest value in multiset $\{f(\bm{x}_i) \mid \bm{x}_i \in \mathcal{S}_-\}$ and $\lfloor.\rfloor$ be the floor function. The constrained optimization problem (\ref{empirical plm}) share the same optimal solution $f^*$ and optimal objective value with the following ranking problem
	\begin{equation}\label{origin rank plm}
	\min_{f}\frac{1}{m}\sum_{i=1}^m \mathbb{I}\left(f(\bm{x}_i^+) - f(\bm{x}_{[\lfloor \tau n \rfloor + 1]}^-) \leq 0\right).
	\end{equation}
\end{prop}
\begin{proof}
	For a fixed $f$, it is clear that the constraint in (\ref{empirical plm}) is equivalent to $b \geq f(\bm{x}_{[\lfloor \tau n \rfloor + 1]}^-)$. Since the objective in (\ref{empirical plm}) is a non-decreasing function of $b$, its minimum is achieved at $b = f(\bm{x}_{[\lfloor \tau n \rfloor + 1]}^-)$. From this, we can transform the original problem (\ref{empirical plm}) into its equivalent form (\ref{origin rank plm}) by substitution.
\end{proof}
Proposition \ref{tcp-rp} reveals the connection between constrained optimization \eqref{empirical plm} and ranking problem \eqref{origin rank plm}. Intuitively, problem \eqref{origin rank plm} makes comparsion between each positve sample and the $(\lfloor \tau n \rfloor + 1)$-th largest negative sample.  Here we give further explanation about its form: it is equivalent to the maximization of the partial-AUC near the risk area.
\begin{figure}[ht]
	\centering
	\includegraphics[width=0.4\linewidth]{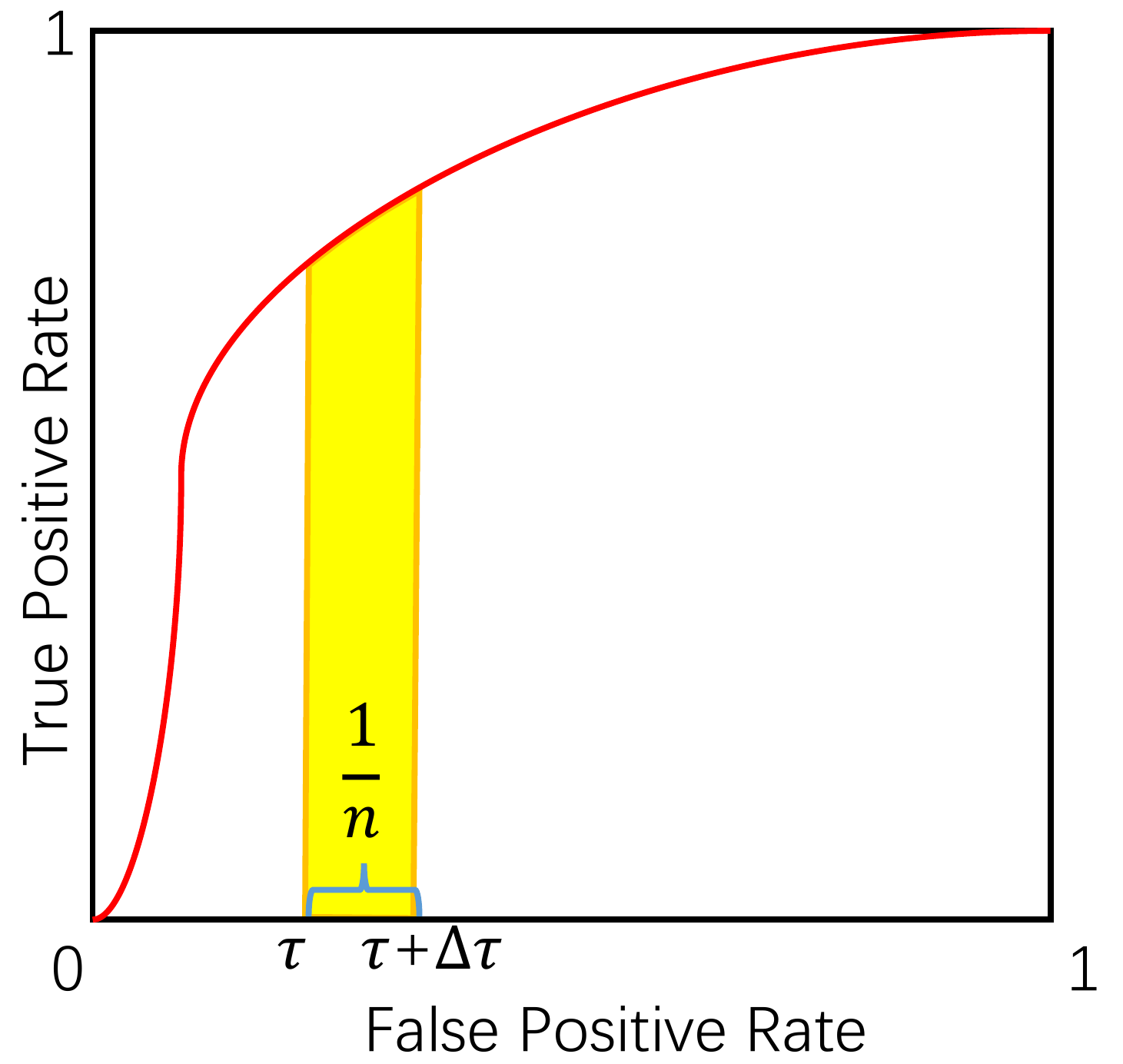}
	\caption{Neyman-Pearson Classification is equivalent to a partial-AUC optimization near the specified FPR.}
	\label{Fig.0}
\end{figure}

Although it seems that partial-AUC optimization considers fewer samples than the full AUC, optimizing (\ref{origin rank plm}) is intractable even if we replace $\mathbb{I}(\cdot)$ by a convex loss, since the operation $[\cdot]$ is non-convex when $\lfloor\tau n\rfloor > 0$. Indeed, Theorem \ref{np-hard} further shows that whatever $\tau$ is chosen, even for some weak hypothesis of $f$, the corresponding optimization problem is NP-hard.
\begin{mydef}\label{def:surr}
	A \emph{surrogate function} of ~$\mathbb{I}(u \leq 0)$ is a contionuous and non-increasing function $L:\mathbb{R} \to \mathbb{R}$ satisfies that (i) $\forall u \leq 0$, $L(u) \geq L(0) \geq 1$; (ii) $\forall u > 0$, $0 \leq L(u) < L(0)$; (iii) $L(u) \rightarrow 0$ as $u \rightarrow +\infty$.
\end{mydef}

\begin{thm}\label{np-hard}
	For fixed $\lambda > 0$, $0 < \tau < 1$ and surrogate function $L(\cdot)$, optimization problem $\tau$-$OPT_{L}^{\lambda}$
	\begin{equation*}
	\min_{f \in \mathcal{H}^d}\frac{1}{m}\sum_{i=1}^m L\left(f(\bm{x}_i^+) - f(\bm{x}_{[\lfloor \tau n \rfloor + 1]}^-)\right)
	\end{equation*} 
	with hyperplane hypothesis set $\mathcal{H}^d = \{f(\bm{x}) = \bm{w}^{\top}\bm{x} \mid \bm{w} \in \mathbb{R}^d, ||\bm{w}|| \leq \lambda\}$ is NP-hard.
\end{thm}

This intractability motivates us to consider the following upper bound approximation of (\ref{origin rank plm}):
\begin{equation}\label{new rank plm}
\min_{f}\frac{1}{m}\sum_{i=1}^m \mathbb{I}\left(f(\bm{x}_i^+) -  \frac{1}{\lfloor \tau n \rfloor + 1}\sum_{i = 1}^{\lfloor \tau n \rfloor + 1} f(\bm{x}_{[i]}^-) \leq 0\right)
\end{equation}
which prefers scores on positive examples to exceed the mean of scores on the worst $\tau$-proportion of negative examples. If $\tau = 0$, optimization (\ref{new rank plm}) is equivalent to the original problem (\ref{origin rank plm}). In general cases, equality also could hold when both the scores of the worst $\tau$-proportion negative examples are the same.

The advantages of considering ranking problem (\ref{new rank plm}) include (details given later): 
\begin{itemize}
	\item For appropriate hypothesis set of $f$, replacing $\mathbb{I}(\cdot)$ by its convex surrogate will produce a tight convex upper bound of the original minimization problem (\ref{origin rank plm}), in contrast to the cost-sensitive classification which may only offer an insecure lower bound. This upper bound is also tighter than the convex relaxation of the original constrained minimization problem \eqref{empirical plm} (see Proposition \ref{prop:upper});
	\item By designing efficient learning algorithm, we can find the global optimal solution of this ranking problem in \emph{linear time}, which makes it well suited for large-scale scenarios, and outperforms most of the traditional ranking algorithms; 
	\item Explicitly takes $\tau$ into consideration and no additional cost hyper-parameters required;
	\item The generalization performance of the optimal solution can be theoretically established.
\end{itemize}

\subsection{Comparison to Alternative Methods}
Before introducing our algorithm, we briefly review some related approaches to FPR constrained learning, and show that our approach can be seen as a \emph{tighter upper bound}, meanwhile maintain linear training complexity.

\textbf{Cost sensitive Learning}~One alternative approach to eliminating the constraint in \eqref{empirical plm} is to approximate it by introducing asymmetric costs for different types of error into classification learning framework:
\begin{equation*}
\min_{f, b} {\cal L}_{emp}^C(f, b) = \frac{1}{m} \sum_{i=1}^m\mathbb{I}\left(f(\bm{x_i^+}) \leq b\right) + \frac{C}{n} \sum_{j=1}^n\mathbb{I}\left(f(\bm{x_j^-}) > b\right),
\end{equation*}
where $C \geq 0$ is a hyper-parameter that punishes the gain of false positive instance. Although reasonable for handling different misclassification costs, we point out that methods under this framework indeed \emph{minimize a lower bound} of problem \eqref{empirical plm}. This can be verified by formulating \eqref{empirical plm} into unconstrained form ${\cal L}_{emp}'(f, b)$
\begin{eqnarray*}
	& &\!\!\!\!{\cal L}_{emp}'(f, b) \\
	&\triangleq&\!\!\!\! \max_{\lambda \geq 0} \frac{1}{m} \sum_{i=1}^m\mathbb{I}\left(f(\bm{x_i^+}) \leq b\right) \!+\! \lambda\left( \frac{1}{n} \sum_{j=1}^n\mathbb{I}\left(f(\bm{x_j^-}) > b\right) - \tau \right) \\
	&\geq&\!\!\!\! \frac{1}{m} \sum_{i=1}^m\mathbb{I}\left(f(\bm{x_i^+}) \leq b\right) + C\left( \frac{1}{n} \sum_{j=1}^n\mathbb{I}\left(f(\bm{x_j^-}) > b\right) - \tau\right)\\
	&=&\!\!\!\! {\cal L}_{emp}^C(f, b) - C\tau.
\end{eqnarray*}
Thus, for a fixed $C$, minimizing ${\cal L}_{emp}^C$ is equivalent to minimize a lower bound of ${\cal L}_{emp}$. In other words, cost-sensitive learning methods are \emph{insecure} in this setting: mismatched $C$ can easily violate the constraint on FPR or make excessive requirements
, and the optimal $C$ for specified $\tau$ is only known after solving the original constrained problem, which can be proved NP-hard. 

In practice, one also has to make a continuous approximation for the constraint in \eqref{empirical plm} for tractable cost-sensitive training, this multi-level approximation makes the relationship between the new problem and original unclear, and blocks any straightforward theoretical justification.

\textbf{Constrained Optimization}~~The main difficulty of employing Lagrangian based method to solve problem \eqref{empirical plm} lies in the fact that both the constraint and objective are non-continuous. In order to make the problem tractable while satisfying FPR constraint strictly, the standard solution approach relies on replacing $\mathbb{I}(\cdot)$ with its convex surrogate function (see Definition \ref{def:surr}) to obtain a convex constrained optimization problem \footnote{For precisely approximating the indicator function in constraint, choosing ``bounded'' loss such as sigmoid or ramp loss \citep{collobert2006trading,gasso2011batch} etc. seems appealing. However, bounded functions always bring about a non-convex property, and corresponding problems are usually NP-hard \citep{yu2012polynomial}. These difficulties limit both efficient training and effective theoretical guarantee of the learned model.}. Interestingly, Proposition \ref{prop:upper} shows that whichever surrogate function and hypothesis set are chosen, the resulting constrained optimization problem is a \emph{weaker upper bound} than that of our approach.
\begin{prop}\label{prop:upper}
	For any non-empty hypothesis set $\mathcal{H}$ and $f \in \mathcal{H}$, convex surrogate function $L(\cdot)$, let
	\begin{eqnarray*}
		\bar{R}_0 \!\!&=&\!\!\! \frac{1}{m}\sum \nolimits_{i=1}^m \mathbb{I}\left(f(\bm{x}_i^+) - f(\bm{x}_{[\lfloor \tau n \rfloor + 1]}^-) \leq 0\right)\\
		\bar{R}_1 \!\!&=&\!\!\! \frac{1}{m} \sum \nolimits_{i=1}^m L\left(f(\bm{x}_i^+) - \frac{1}{\lfloor \tau n \rfloor + 1}\sum_{j=1}^{\lfloor \tau n \rfloor + 1} f(\bm{x}_{[j]}^-)\right)\\
		\bar{R}_2 \!\!&=&\!\!\! \min_{b \in \mathbb{R}} \frac{1}{m} \sum_{i=1}^m L\left(f(\bm{x}_i^+) - b\right)~\textrm{s.t.}~\frac{1}{n} \sum_{j=1}^n L\left(b - f(\bm{x}_j^-)\right) \leq \tau,
	\end{eqnarray*} 
	we have 
	\begin{equation*}
	\bar{R}_0 \leq \bar{R}_1 \leq \bar{R}_2.
	\end{equation*}
\end{prop}

Thus, in general, there is an exact gap between our risk $\bar{R}_1$ and convex constrained optimization risk $\bar{R}_2$.  In practice one may prefer a tighter approximation since it represents the original objective better. Moreover, our Theorem \ref{bound p} also achieves a tighter bound on the generalization error rate, by considering empirical risk $\bar{R}_1$.

\textbf{Ranking-Thresholding}~Traditional bipartite ranking methods usually have super-linear training complexity. Compared with them, the main advantage of our algorithm comes from its linear-time conplexity in each iteration without any convergence rate (please refer to Table \ref{tabcomp}). We named our algorithm $\tau$-FPL, and give a detailed description of how it works in the next sections.

\section{Tolerance-Constrained False Positive Learning}
Based on previous discussion, our method will be divided into two stages, namely \emph{scoring} and \emph{thresholding}. In scoring, a function $f(\cdot)$ is learnt to maximize the probability of giving higher a score to positive instances than the centroid of top $\tau$ percent of the negative instances. In thresholding, a suitable threshold $b$ will be chosen, and the final prediction of an instance $\bm{x}$ can be obtained by
\begin{equation}
y = \mathrm{sgn}\left(f(\bm{x}) - b\right) \ .
\end{equation}

\subsection{Tolerance-Constrained Ranking}
In (\ref{new rank plm}), we consider linear scoring function $f(\bm{x}) = \bm{w}^{\top}\bm{x}$, where $\bm{w} \in \mathbb{R}^d$ is the weight vector to be learned, $\text{L}_2$ regularization and replace $\mathbb{I}(u < 0)$ by some of its convex surrogate $l(\cdot )$. Kernel methods can be used for nonlinear ranking functions. As a result, the learning problem is
\begin{equation}\label{learning problem}
\min \limits_{\bm{w}} \frac{1}{m}\sum \nolimits_{i=1}^m l\left(\bm{w}^{\top}\bm{x}_i^+ - \frac{1}{k}\sum \nolimits_{j=1}^k\bm{w}^{\top}\bm{x}_{[j]}^-\right) + \frac{R}{2}\|\bm{w}\|^2,
\end{equation}
where $R > 0$ is the regularization parameter, and $k = \lfloor \tau n \rfloor + 1$.

Directly minimizing (\ref{learning problem}) can be challenging due to the $[\cdot]$ operator, we address it by considering its dual,
\begin{thm}\label{dual form}
	Define $\bm{X}^+=[\bm{x}_1^+, ..., \bm{x}_m^+]^{\top}$ and $\bm{X}^-=[\bm{x}_1^-, ..., \bm{x}_m^-]^{\top}$ be the matrix containing positive and negative instances in their rows respectively, the dual problem of (\ref{learning problem}) can be written by
	\begin{eqnarray}\label{dual problem}
	\min \limits_{(\bm{\alpha}, \bm{\beta} \in \Gamma_k)}\! g(\bm{\alpha}, \bm{\beta}) \!\! = \!\!\frac{1}{2mR}||\bm{\alpha}^{\top}\bm{X}^+ \!-\! \bm{\beta}^{\top}\bm{X}^-||^2 + \sum_{i=1}^m l^*(-\alpha_i),
	\end{eqnarray}
	where $\bm{\alpha}$ and $\bm{\beta}$ are dual variables, $l^*(\cdot)$ is the convex conjugate of $ l(\cdot)$, and the domain $\Gamma_k$ is defined as
	\begin{eqnarray}\label{dual set}
	\nonumber \Gamma_k = \left\{\bm{\alpha} \in \mathbb{R}_+^m, \bm{\beta} \in \mathbb{R}_+^n \mid \sum_{i=1}^m \alpha_i = \sum_{j=1}^n \beta_j; \beta_j \leq \frac{1}{k}\sum_{i=1}^m \alpha_i, \forall j\right\}.
	\end{eqnarray}
	Let $\bm{\alpha}^*$ and $\bm{\beta}^*$ be the optimal solution of (\ref{dual problem}), the optimal
	\begin{equation}
	\bm{w}^* = (mR)^{-1}(\bm{\alpha}^{*\top}\bm{X}^+ - \bm{\beta}^{*\top}\bm{X}^-)^{\top}. 
	\end{equation}
\end{thm}

According to Theorem 1, learning scoring function $f$ is equivalent to learning the dual variables $\bm{\alpha}$ and $\bm{\beta}$ by solving problem \eqref{dual problem}. Its optimization naturally falls into the area of projected gradient method.  Here we choose $l(u)=[1-u]_+^2$ where $[x]_+ \triangleq \max\{x, 0\}$ due to its simplicity of conjugation. The key steps are summarized in Algorithm \ref{alg:t-FPL}. At each iteration, we first update solution by the gradient of the objective function $g(\bm{\alpha}, \bm{\beta})$, then project the  dual solution onto the feasible set $\Gamma_k$. In the sequel, we will show that this projection problem can be efficiently solved in {\it linear time}. In practice, since $g(\cdot)$ is smooth, we also leverage Nesterov's method to further accelerate the convergence of our algorithm. Nesterov's method \citep{Nesterov03} achieves $O(1/T^2)$ convergence rate for smooth objective function, where $T$ is the number of iterations.

\begin{algorithm}[htb!]
	\caption{$\tau$-FPL Ranking} \label{alg:t-FPL}
	\begin{algorithmic}[1]
		\REQUIRE $X^+ \in \mathbb{R}^{m \times d}, X^- \in \mathbb{R}^{n \times d}$, maximal FPR tolerance $\tau$, regularization parameter $R$, stopping condition $\epsilon$
		\STATE
		Randomly initialize $\bm{\alpha_0}$ and $\bm{\beta_0}$\\
		\STATE Set counter: $t \leftarrow 0$\\
		\WHILE{$t = 0$ \textbf{or} $|g(\bm{\alpha_{t}}, \bm{\beta_{t}}) - g(\bm{\alpha_{t-1}}, \bm{\beta_{t-1}})| > \epsilon$} 
		\STATE Compute gradient of $g(\cdot)$ at point $(\bm{\alpha_t}, \bm{\beta_t})$\\
		\STATE Compute $\bm{\alpha_{t+1}'}, \bm{\beta_{t+1}'}$ by gradient descent;\\
		\STATE Project $\bm{\alpha_{t+1}'}, \bm{\beta_{t+1}'}$ onto the feasible set $\Gamma_k$:
		\begin{equation*}
		(\bm{\alpha_{t+1}}, \bm{\beta_{t+1}}) \leftarrow \Pi_{\Gamma_k}(\bm{\alpha_{t+1}'}, \bm{\beta_{t+1}'})
		\end{equation*}
		\STATE Update counter: $t \leftarrow t+1$;
		\ENDWHILE
		\STATE Return $\bm{w} \leftarrow (mR)^{-1}(\bm{\alpha_{t}}^{\top}X^+ - \bm{\beta_{t}}^{\top}X^-)^{\top}$
	\end{algorithmic}
\end{algorithm}

\subsection{Linear Time Projection onto the Top-k Simplex}\label{proj sec}
One of our main technical results is a \emph{linear time} projection algorithm onto $\Gamma_k$, even in the case of $k$ is close to $n$. For clear notations, we reformulate the projection problem as
\begin{equation}\label{proj plm}
\begin{split}
\min_{\bm{\alpha}    \geq 0, \bm{\beta} \geq 0} ~ & ~ \frac{1}{2}||\bm{\alpha} - \bm{\alpha}^0||^2 + \frac{1}{2}||\bm{\beta} - \bm{\beta}^0||^2\\
\text{\small s.t.}~ & ~\sum \nolimits_{i=1}^m \alpha_i = \sum \nolimits_{j=1}^n \beta_j, \quad \beta_j \leq \frac{1}{k}\sum \nolimits_{i=1}^m \alpha_i, \forall j.
\end{split}
\end{equation}
It should be noted that, many Euclidean projection problems studied in the literature can be seen as a special case of this problem. If the term $\sum_{i=1}^m \alpha_i$ is fixed, or replaced by a constant upper bound $C$, we obtain a well studied case of \emph{continuous quadratic knapsack problem} (CQKP)
\begin{align*}
\min_{\bm{\beta}} ~~ ||\bm{\beta} - \bm{\beta}^0||^2 \quad
\textrm{\small s.t.} ~\sum \nolimits_{i=1}^n \beta_i \ \leq\ C, 0\leq \beta_i \leq C_1 \ ,
\end{align*}
where $C_1 = C/k$. Several efficient methods based on median-selecting or variable fixing techniques are available \citep{patriksson2008survey}. On the other hand, if $k=1$, all  upper bounded constraints are automatically satisfied and can be omitted. Such special case has been studied,  for example, in \citep{liu2009efficient} and \citep{li2014top}, both of which achieve $O(n)$ complexity.

Unfortunately, none of those above methods can be directly applied to solving the generalized case (\ref{proj plm}), due to its property of unfixed upper-bound constraint on $\beta$ when $k >1$. To our knowledge, the only attempt to address the problem of unfixed upper bound is \citep{lapin2015top}. They solve a similar (but simpler) problem
\[
\min_{\beta} ~~ ||\bm{\beta} - \bm{\beta}^0||^2 ~~ \textrm{\small s.t.}~~ 0 \leq \beta_j \leq \frac{1}{k}\sum \nolimits_{i=1}^n \beta_i
\]
based on sorting and exhaustive search and their method achieves a runtime complexity $O(n\log(n) + kn)$, which is super-linear and even quadratic when $k$ and $n$ are linearly dependent. By contrast, our proposed method can be applied to both of the aforementioned special cases with minor changes and remains $O(n)$ complexity. The notable characteristic of our method is the efficient combination of bisection and divide-and-conquer: the former offers the guarantee of worst complexity, and the latter significantly reduces the large constant factor of bisection method.

We first introduce the following theorem, which gives a detailed description of the solution for (\ref{proj plm}).
\begin{thm}\label{sol of proj}
	$(\bm{\alpha}^*\in \mathbb{R}^m, \bm{\beta}^*\in \mathbb{R}^n)$ is the optimal solution of (\ref{proj plm}) if and only if there exist dual variables $C^* \geq 0$, $\lambda^*$, $\mu^* \in \mathbb{R}$ satisfy the following system of linear constraints:
	\begin{eqnarray}
	C^*&=&\sum \nolimits_{i=1}^m [\alpha_i^0 - \lambda^*]_+ \label{def C}\\
	C^*&=&\sum\nolimits_{j=1}^n \min\{[\beta_j^0 - \mu^*]_+, C^*/k\}\label{def mu}\\
	0&=&\lambda^* + \mu^* + \frac{1}{k} \sum\nolimits_{j=1}^n\left[ \beta_j^0 - \mu^* - C^*/k \right]_+ \label{3 var}
	\end{eqnarray}
	and $\alpha_i^* = [\alpha_i^0 - \lambda^*]_+ \label{def alpha}$, $\beta_j^* = \min\{[\beta_j^0 - \mu^*]_+, {C^*}/{k}\} \label{def beta}$.
\end{thm}

Based on Theorem \ref{sol of proj}, the projection problem can be solved by finding the value of three dual variables $C$, $\lambda$ and $\mu$ that satisfy the above linear system. Here we first propose a basic bisection method which guarantees the worst time complexity. Similar method has also been used in \citep{liu2009efficient}. For brevity, we denote $\alpha_{[i]}^0$ and $\beta_{[i]}^0$ the $i$-largest dimension in $\bm{\alpha^0}$ and $\bm{\beta^0}$ respectively, and define function $C(\lambda)$, $\mu(C)$, $\delta(C)$ and $f(\lambda)$ as follows\footnote{Indeed, for some $C$, $\mu(C)$ is not one-valued and thus need more precise definition. Here we omit it for brevity, and leave details in section \ref{redifine_muc}.}:
\begin{eqnarray}
C(\lambda) &=& \sum \nolimits_{i=1}^m [\alpha_i^0 - \lambda]_+ \label{f:C}\\
\mu(C) &=& \mu \text{ satisfies } (\ref{def mu})\label{f:mu}\\
\delta(C) &=& \mu(C) + {C}/{k} \label{f:delta}\\
f(\lambda) &=& k\lambda + k\mu(C(\lambda)) + \sum \nolimits_{j=1}^n[\beta_j^0 - \delta(C(\lambda))]_+. \label{f:f}
\end{eqnarray}

The main idea of leveraging bisection to solve the system in theorem \ref{sol of proj} is to find the root of $f(\lambda) = 0$. In order to make bisection work, we need three conditions: $f$ should be continuous; the root of $f$ can be efficiently bracketed in a interval; and the value of $f$ at the two endpoints of this interval have opposite signs. Fortunately, based on the following three lemmas, all of those requirements can be satisfied.
\begin{lemma}\label{c=0}
	(Zero case) \quad $(\bm{0}^m, \bm{0}^n)$ is an optimal solution of (\ref{proj plm}) if and only if $k\alpha_{[1]}^0 + \sum_{j=1}^k \beta_{[j]}^0 \leq 0$.
\end{lemma}

\begin{lemma}\label{bound}
	(Bracketing $\lambda^*$)\quad If $C^* > 0$, $\lambda^* \in (-\beta_{[1]}^0,\ \alpha_{[1]}^0)$.
\end{lemma}

\begin{lemma}\label{inc-dec}
	(Monotonicity and convexity)
	\begin{enumerate}
		\item $C(\lambda)$ is convex, continuous and strictly decreasing in $(-\infty$, $\alpha_{[1]}^0)$; \label{c1}
		\item $\mu(C)$ is continuous, monotonically decreasing in $(0, +\infty)$; \label{c2}
		\item $\delta(C)$ is continuous, strictly increasing in $(0, +\infty)$; \label{c3}
		\item $f(\lambda)$ is continuous, strictly increasing in $(-\infty, \alpha_{[1]}^0)$. \label{c4}
	\end{enumerate}
	Furthermore, we can define the inverse function of $C(\lambda)$ as $\lambda(C)$, and rewrite $f(\lambda)$ as:
	\begin{equation}
	f(\lambda(C)) = k\lambda(C) + k\mu(C) + \sum \nolimits_{j=1}^n[\beta_j^0 - \delta(C)]_+,
	\end{equation}
	it is a convex function of $C$, strictly decreasing in $(0, +\infty)$.
\end{lemma}

Lemma \ref{c=0} deals with the special case of $C^*=0$. Lemma \ref{bound} and \ref{inc-dec} jointly ensure that bisection works when $C^* > 0$; Lemma \ref{bound} bounds $\lambda^*$;  Lemma \ref{inc-dec} shows that $f$ is continuous, and since it is also strictly increasing, the value of $f$ at two endpoints must has opposite sign.

\textbf{Basic method: bisection \& leverage convexity } We start from select current $\lambda$ in the range $(-\beta_{[1]}^0,\ \alpha_{[1]}^0) \triangleq [l, u]$. Then compute corresponding $C$ by (\ref{f:C}) in $O(m)$, and use the current $C$ to compute $\mu$ by (\ref{f:mu}). Computing $\mu$ can be completed in $O(n)$ by a well-designed median-selecting algorithm \citep{kiwiel2007linear}. With the current (i.e. updated) $C$, $\lambda$ and $\mu$ in hand, we can evaluate the sign of $f(\lambda)$ in $O(n)$ and determine the new bound of $\lambda$. In addition, the special case of $C=0$ can be checked using Lemma \ref{c=0} in $O(m + n)$ by a linear-time k-largest element selecting algorithm \citep{kiwiel2005floyd}. Since the bound of $\lambda$ is irrelevant to $m$ and $n$, the number of iteration for finding $\lambda^*$ is $\log(\frac{u - l}{\epsilon})$, where $\epsilon$ is the maximum tolerance of the error. Thus, the worst runtime of this algorithm is $O(m +n)$. Furthermore, we also leverage the convexity of $f(\lambda(C))$ and $C(\lambda)$ to further improve this algorithm, please refer to \citep{liu2009efficient} for more details about related techniques.

Although bisection solves the projections in linear time, it may lead to a slow convergence rate. We further improve the runtime complexity by reducing the constant factor $\log(\frac{u - l}{\epsilon})$. This technique benefits from exploiting the monotonicity of both functions $C(\lambda)$, $\mu(C)$, $\delta(C)$ and $f(\lambda)$, which have been stated in Lemma \ref{inc-dec}. Notice that, our method can also be used for finding the root of arbitary piecewise linear and monotone function, without the requirement of convexity.

\begin{algorithm}[htb!]
	\caption{Linear-time Projection on Top-k simplex} \label{alg:proj}
	\begin{algorithmic}[1]
		\REQUIRE $\bm{\alpha^0} \in \mathbb{R}^m, \bm{\beta^0} \in \mathbb{R}^n$, maximal accuracy $\epsilon$
		\STATE Calculate initial uncertainly intervals for $\lambda$, $C$, $\delta$ and $\mu$;
		\STATE Initialize breakpoint caches for $C(\lambda)$, $\mu(C)$, $\delta(C)$, $f(\lambda)$: 
		\begin{equation*}
		Cache_C \leftarrow \{\alpha_i^0 \mid \forall i\},~~ Cache_{\mu},  Cache_{\delta}, Cache_f \leftarrow \{\beta_j^0 \mid \forall j\}
		\end{equation*}
		\STATE Initialize partial sums of $C(\lambda)$, $\mu(C)$, $\delta(C)$, $f(\lambda)$ with zero;
		\STATE Set $t \leftarrow 0$ and $\lambda_0 \leftarrow (\alpha_{[1]}^0 - \beta_{[1]}^0)/2$;
		\WHILE{$t = 0$ \textbf{or} $|\lambda_{t} - \lambda_{t-1}| > \epsilon$}
		\STATE Calculate $C_t$, $\mu_t$, $\delta_t$, $f_t$(by leveraging corresponding caches and partial sums);
		\STATE Prune caches and update partial sums;
		\STATE Shrink intervals of $\lambda$, $C$, $\delta$ and $\mu$ based on sign of $f(\lambda_t)$;
		\STATE $t \leftarrow t + 1$;
		\STATE Set $\lambda_t$ as midpoint of current new interval
		\ENDWHILE
		\STATE Return $\lambda^* \leftarrow \lambda_t$, $\mu^* \leftarrow \mu_t$, $C^* \leftarrow C_t$
	\end{algorithmic}
\end{algorithm}

\textbf{Improved method: endpoints Divide \& Conquer } Lemma \ref{inc-dec} reveals an important chain monotonicity between the dual variables, which can used to improve the performance of our baseline method. The key steps are summarized in Algorithm \ref{alg:proj}. Denote the value of a variable $z$ in iteration $t$ as $z_t$. For instance, if $\lambda_t > \lambda_{t-1}$, from emma \ref{inc-dec} we have $C_t < C_{t-1}$, $\mu_t > \mu_{t-1}$ and $\delta_t < \delta_{t - 1}$.
This implies that we can set uncertainty intervals for both $\lambda$, $C$, $\mu$ and $\delta$. As the interval of $\lambda$ shrinking, lengths of these four intervals can be reduced simultaneously. On the other hand, notice that $C(\lambda)$ is indeed piecewise linear function (at most $m + 1$ segments), the computation of its value only contains a comparison between $\lambda_t$ and all of the $\alpha_i^0$s. By keeping a cache of $\alpha_i^0$s and discard those elements which are out of the current bound of $\lambda$ in advance, in each iteration we can reduce the expected comparison counts by half. A more complex but similar procedure can also be applied for computing $\mu(C)$, $\delta(C)$, and $f(\lambda)$, because both of these functions are piecewise linear and the main cost is the comparison with $O(m+n)$ endpoints.  As a result, for approximately linear function and evenly distributed breakpoints, if the first iteration of bisection costs $\gamma (m+n)$ time, the overall runtime of the projection algorithm will be $\gamma (m+n) + \gamma (m+n)/2 +... \leq 2\gamma (m+n)$, which is much less than the original bisection algorithm whose runtime is $\log(\frac{u - l}{\epsilon})\gamma (m+n)$. 

\subsection{Convergence and Computational Complexity}

Following immediately from the convergence result of Nesterov's method, we have:
\begin{thm}\label{conv rate}
	Let $\bm{\alpha}_T$ and $\bm{\beta}_T$ be the output from the $\tau$-FPL algorithm after T iterations, then $g(\bm{\alpha}_T,\bm{\beta}_T) \leq \min g(\bm{\alpha}, \bm{\beta}) + \epsilon$, where $T \geq O(1/\sqrt{\epsilon})$.
\end{thm}

\begin{table}[htb]
	\centering
	{
		\begin{tabular}{lcc}
			\hline
			\textbf{Algorithm} & \textbf{Training} & \textbf{Validation} \\
			\hline\hline
			$\tau$-FPL &$O(d(m+n)/T^2)$ & Linear\\\\
			
			TopPush &$O(d(m+n)/T^2)$ & Linear \\\\
			
			CS-SVM &$O(d(m+n)/T)$ & Quadratic \\\\
			
			$\text{SVM}_{\text{tight}}^{\text{pAUC}}$ &$O((m\log m+ n\log n + d(m+n))/T)$ & Linear\\\\
			
			Bipartite & $O((d(m+n) + (m+n)\log(m+n))/T)$ &\multirow{2}{*}{Linear}\\
			Ranking &  $\sim O(dmn + mn\log(mn)/\sqrt T)$ & \multirow{2}{*}{}\\
			\hline
		\end{tabular}
		\caption{Complexity comparison with SOTA approaches}
		\label{tabcomp}}
\end{table}
Finally, the computational cost of each iteration is dominated by the gradient evaluation and the projection step. Since the complexity of projection step is $O(m + n)$ and the cost of computing the gradient is $O(d(m + n))$, combining with Theorem \ref{conv rate} we have that: to find an $\epsilon$-suboptimal solution, the total computational complexity of $\tau$-FPL is $O(d(m+n)/\sqrt{\epsilon})$. Table \ref{tabcomp} compares the computational complexity of $\tau$-FPL with that of some state-of-the-art methods.  The order of validation complexity corresponds to the number of hyper-parameters. From this, it is easy to see that $\tau$-FPL is asymptotically more efficient.

\subsection{Out-of-Bootstrap Thresholding}\label{OOB}
In the thresholding stage, the task is to identify the boundary between the positive instances and $(1-\tau)$ percent of the negative instances. Though thresholding on the training set is commonly used in \citep{joachims1996probabilistic,davenport2010tuning,scheirer2013toward}, it may result in overfitting. Hence, we propose an out-of-bootstrap method to find a more accurate and stable threshold. At each time, we randomly split the training set into two sets $\mathcal{S}_1$ and $\mathcal{S}_2$, and then train on $\mathcal{S}_1$ as well as the select threshold on $\mathcal{S}_2$. The procedure can be running multiple rounds to make use of all the training data. Once the process is completed, we can obtain the final threshold by averaging. On the other hand, the final scoring function can be obtained by two ways: learn a scoring function using the full set of training data, or gather the weights learned in each previous round and average them. This method combines both the advantages of out-of-bootstrap and soft-thresholding techniques: accurate error estimation and reduced variance with little sacrifice on the bias, thus fits the setting of thresholding near the risk area.

\section{Theoretical Guarantees}\label{sec:theory}
Now we develop the theoretical guarantee for the scoring function, which bounds the probability of giving any positive instances higher score than $1-\tau$ proportion of negative instances. To this end, we first define $h(\bm{x}, f)$, the probability for any negative instance to be ranked above $x$ using $f$, i.e. $h(x, f) = \mathbb{E}_{x^- \sim \mathbb{P}^-}[\mathbb{I}(f(\bm{x}) \leq f(\bm{x}^-))]$, and then measure the quality of $f$ by $P(f, \tau) = \mathbb{P}_{\bm{x}^+ \sim \mathbb{P}^+}(h(\bm{x}^+, f) \geq \tau)$, which is the probability of giving any positive instances lower score than $\tau$ percent of negative instances. The following theorem bounds $P(f, \tau)$ by the empirical loss $L_{\bar{k}}$.
\begin{thm}\label{bound p}
	Given training data $\mathcal{S}$ consisting of $m$ independent instances from distribution $\mathbb{P}^+$ and $n$ independent instances from distribution $\mathbb{P}^-$, let $f^*$ be the optimal solution to the problem (\ref{learning problem}). Assume $m \geq 12$ and $n \gg s$. We have, for proper $R$ and any $k \leq n$, with a probability at least $1-2e^{-s}$,
	\begin{equation}\label{p bound eq}
	P(f^*, \tau) \leq L_{\bar{k}}+ O\left(\sqrt{(s + \log(m)/m)}\right),
	\end{equation}
	where $\tau = k/n + O\left(\sqrt{\log m/n}\right)$, and $L_{\bar{k}} = \frac{1}{m} \sum_{i=1}^m l\left(f^*(\bm{x}_i^+) - \frac{1}{k}\sum_{j=1}^k f^*(\bm{x}_{[j]}^-)\right)$.
\end{thm}

Theorem \ref{bound p} implies that if $L_{\bar{k}}$ is upper bounded by $O(\log(m)/m))$, the probability of ranking any positive samples below $\tau$ percent of negative samples is also bounded by $O(\log(m)/m))$. If $m$ is approaching infinity, $P(f^*, \tau)$ would be close to 0, which means in that case, we can almost ensure that by thresholding at a suitable point, the true-positive rate will get close to 1. Moreover, we observe that $m$ and $n$ play different roles in this bound. For instance, it is well known that the largest absolute value of Gaussian random instances grows in $\log(n)$. Thus we believe that the growth of $n$ only slightly affects both the largest and the centroid of top-proportion scores of negatives samples. This leads to a conclusion that increasing $n$ only slightly raise $L_{\bar{k}}$, but significant reduce the margin between target $\tau$ and $k/n$. On the other hand, increasing $m$ will reduce upper bound of $P$, thus increasing the chance of finding positive instances at the top. In sum, $n$ and $m$ control $\tau$ and $P$ respectively.\\

\section{Experiment Results}\label{sec:experiment}
\subsection{Effectiveness of the Linear-time Projection}
We first demonstrate the effectiveness of our projection algorithm. Following the settings of \citep{liu2009efficient}, we randomly sample 1000 samples from the normal distribution ${\cal N}(0,1)$ and solve the projection problem. The comparing method is \emph{ibis} \citep{liu2009efficient}, an improved bisection algorithm which also makes use of the convexity and monotonicity. All experiments are running on an Intel Core i5 Processor. As shown in Fig.\ref{Fig.lable}, thanks to the efficient reduction of the constant factor, our method outperforms \emph{ibis} by saving almost $75\%$ of the running time in the limit case.

We also solve the projection problem proposed in \citep{lapin2015top} by using a simplified version of our method, and compare it with the method presented in \citep{lapin2015top} (PTkC), whose complexity is $O(n\log(n) + kn)$. As one can observe from Fig.\ref{Fig.sub.2}, our method is linear in complexity regarding with $n$ and does not suffer from the growth of $k$. In the limit case (both large $k$ and $n$), it is more than three-order of magnitude faster than the competitors.

\begin{figure}[tb!]
	\centering
	\subfigure[Run time against the method ibis (log-log).]{
		\label{Fig.sub.1}
		\includegraphics[width=0.45\linewidth]{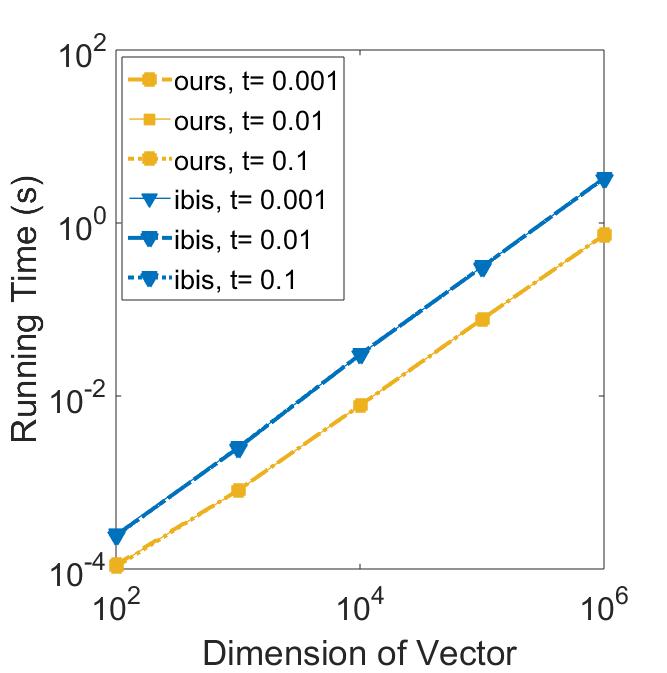}}
	\subfigure[Run time against PTkC (log-log).]{
		\label{Fig.sub.2}
		\includegraphics[width=0.45\linewidth]{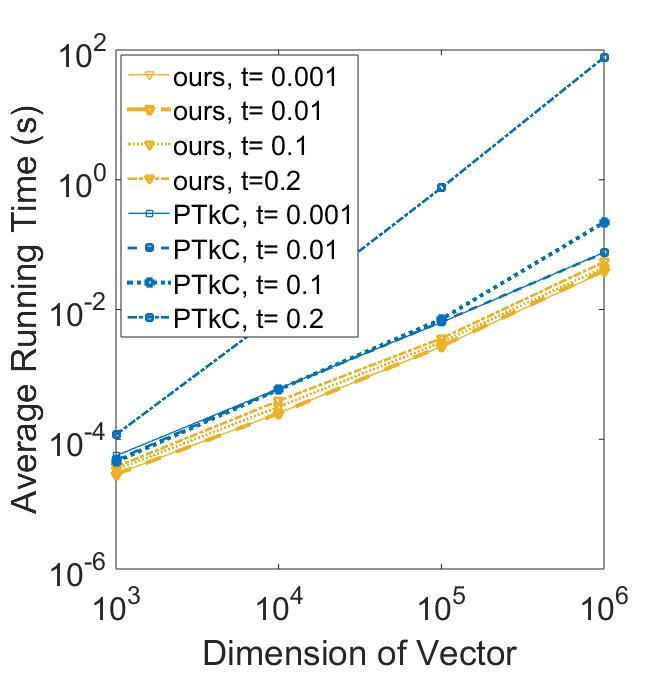}}
	\caption{Running time against two peer projection methods.}
	\label{Fig.lable}
\end{figure}

\subsection{Ranking Performance}\label{RP}

\begin{table*}[htb!]\small\centering\resizebox{1\textwidth}{!}{\begin{tabular}{c|cc|ccccc|ccccc|cccccc}\hline\multirow{2}{*}{\textbf{Dataset}}  &  \multicolumn{2}{c|}{\textbf{heart}}  &  \multicolumn{5}{c|}{\textbf{spambase}}  &  \multicolumn{5}{c|}{\textbf{real-sim}}  &  \multicolumn{6}{c}{\textbf{w8a}}  \\  &  \multicolumn{2}{c|}{120/150,d:13}  &  \multicolumn{5}{c|}{1813/2788,d:57}  &  \multicolumn{5}{c|}{22238/50071,d:20958}  &  \multicolumn{6}{c}{2933/62767,d:300}\\ \hline$\tau(\%)$ & 5 & 10 & 0.1 & 0.5 & 1 & 5 & 10 & 0.01 & 0.1 & 1 & 5 & 10 & 0.05 & 0.1 & 0.5 & 1 & 5 & 10 \\\hline\hline\textbf{CS-SVM}& .526 & .691 & .109 & .302 & $\bm{.487}$ & .811 & .920 & .376 & $\bm{.748}$ & .921 & .972 & .990 & .501 & .520 & .649 & .695 & .828 & .885 \\\textbf{TopPush}& $\bm{.541}$ & .711 & .112 & .303 & .484 & .774 & .845 & $\bm{.391}$ & .747 & .920 & .968 & .983 & $\bm{.508}$ & $\bm{.551}$ & .627 & .656 & .761 & .842 \\\textbf{$\text{SVM}_{\text{tight}}^{\text{pAUC}}$}& .509 & .728 & N/A & N/A & N/A & N/A & N/A & N/A & N/A & N/A & N/A & N/A & N/A & N/A & N/A & N/A & N/A & N/A \\\textbf{$\tau$-Rank}&$\bm{.541}$&$\bm{  .740}$&$\bm{  .112}$&$\bm{  .305}$&$  .460$&$\bm{  .842}$&$\bm{  .929}$&$\bm{  .391}$&$.747$&$ .920$&$\bm{  .975}$&$\bm{  .991}$&$\bm{ .508 }$&$\bm{ .551 }$&$.645$&$\bm{ .710 }$&$\bm{ .832 }$&$\bm{ .894}$\\\textbf{2$\tau$-Rank}&$\bm{.547}$&$\bm{  .734}$&$\bm{  .112}$&$\bm{  .311}$&$  .477$&$\bm{  .862}$&$\bm{  .936}$&$\bm{  .391}$&$.747$&$\bm{  .922}$&$\bm{  .978}$&$\bm{  .992}$&$\bm{ .508 }$&$.549$&$\bm{ .675 }$&$\bm{ .739 }$&$\bm{ .841 }$&$\bm{ .902}$\\\hline\end{tabular}}\caption{Ranking performance by different values of the tolerance $\tau$. The number of positive/negative instances and feature dimensions (`d') is shown together with the name of each dataset. The best results are shown in bold. 'N/A's denote the experiments that require more than one week for training.}\label{tab:real_exp_rank}\end{table*}

Next, we validate the ranking performance of our $\tau$-FPL method, i.e. scoring and sorting test samples, and then evaluate the proportion of positive samples ranked above $1- \tau$ proportion of negative samples. Considering ranking performance independently can avoid the practical problem of mismatching the constraint in (\ref{empirical plm}) on testing set, and always offer us the optimal threshold.

Specifically, we choose (\ref{origin rank plm}) as evaluation and validation criterion. Compared methods include cost-sensitive SVM (CS-SVM) \citep{osuna1997support}, which has been shown a lower bound approximation of (\ref{origin rank plm}); TopPush \citep{li2014top} ranking, which focus on optimizing the absolute top of the ranking list, also a special case of our model ($\tau = 0$); $\text{SVM}_{\text{tight}}^{\text{pAUC}}$ \citep{Narasimhan2013ASS}, a more general method which designed for optimizing arbitrary partial-AUC. We test two version of our algorithms: $\tau$-Rank and $2\tau$-Rank, which correspond to the different choice of $\tau$ in learning scheme. Intuitively, enlarge $\tau$ in training phase can be seen as a top-down approximation---from upper bound to the original objective (\ref{empirical plm}). On the other hand,  the reason for choosing $2\tau$ is that, roughly speaking, the average score of the top $2\tau$ proportion of negative samples may close to the score of $\tau n$-th largest negative sample.

\textbf{Settings}. We evaluate the performance on publicly benchmark datasets with different domains and various sizes\footnote{\url{https://www.csie.ntu.edu.tw/~cjlin/libsvmtools/datasets/binary}}. For small scale datasets($\leq 10,000$ instances), 30 times stratified hold-out tests are carried out, with $2/3$ data as train set and $1/3$ data as test set. For large datasets, we instead run 10 rounds. In each round, hyper-parameters are chosen by 5-fold cross validation from grid, and the search scope is extended if the optimal is at the boundary.

\textbf{Results}. Table \ref{tab:real_exp_rank} reports the experimental results. We note that at most cases, our proposed method outperforms other peer methods. It confirms the theoretical analysis that our methods can extract the capacity of the model better. For TopPush, it is highly-competitive in the case of extremely small $\tau$, but gradually lose its advantage as $\tau$ increase. The algorithm of $\text{SVM}_{\text{tight}}^{\text{pAUC}}$ is based on cutting-plane methods with exponential number of constraints, similar technologies are also used in many other ranking or structured prediction methods, e.g. Structured SVM \citep{tsochantaridis2005large}. The time complexity of this kind of methods is $O((m+n)\log(m+n))$, and we found that even for thousands of training samples, it is hard to finish experiments in allowed time. 

\subsection{Overall Classification Accuracy}
\begin{table*}[htb!]\small
	\centering
	\resizebox{1\textwidth}{!}
	{
		\begin{tabular}{c| c| cccc |c ccc}
			\hline
			\textbf{Dataset(+/-)} &  \textbf{$\tau$(\%)} & \textbf{BS-SVM} &  \textbf{CS-LR} &  \textbf{CS-SVM} &  \textbf{CS-SVM-OOB} &  \textbf{$\tau$-FPL} &  \textbf{$2\tau$-FPL}\\
			\hline
			\hline
			\textbf{heart}&5&$(.069, .675), .713$&$(.035, .394), .606$&$(.027, .327), .673$&$(.058, .553), .609$&$(.053, .582), \bm{.468}$&$(.055, .584), \bm{.514}$\\
			120/150,d:13&10&$(.121, .774), .435$&$(.058, .615), .385$&$(.078, .666), .334$&$(.088, .682), .318$&$(.086, .686), \bm{.314}$&$(.080, .679), \bm{.317}$\\
			\hline
			
			\textbf{breast-cancer}&1&$(.015, .964), .559$&$(.007, .884), \bm{.116}$&$(.006, .870), .130$&$(.014, .955), .451$&$(.013, .955), .324$&$(.011, .949), .192$\\
			239/444&5&$(.063, .978), .276$&$(.013, .965), .035$&$(.017, .965), .034$&$(.046, .974), \bm{.026}$&$(.041, .976), \bm{.025}$&$(.045, .974), \bm{.026}$\\
			d:10&10&$(.113, .985), .142$&$(.035, .970), .030$&$(.044, .973), .027$&$(.095, .981), .020$&$(.098, .982), \bm{.018}$&$(.094, .982), \bm{.018}$\\
			\hline
			
			
			\textbf{spambase}&0.5&$(.008, .426), 1.220$&$(.007, .011), 1.362$&$(.002, .109), .891$&$(.005, .275), .790$&$(.005, .278), \bm{.722}$&$(.004, .268), \bm{.732}$\\
			1813/2788&1&$(.013, .583), .748$&$(.007, .011), .989$&$(.004, .256), .744$&$(.009, .418), .582$&$(.008, .416), .584$&$(.008, .440), \bm{.560}$\\
			d:57&5&$(.054, .895), .192$&$(.007, .011), .989$&$(.020, .667), .333$&$(.047, .793), .207$&$(.041, .822), \bm{.178}$&$(.046, .845), \bm{.155}$\\
			&10&$(.103, .941), .087$&$(.007, .011), .989$&$(.051, .716), .284$&$(.090, .902), .099$&$(.087, .925), \bm{.075}$&$(.090, .928), \bm{.072}$\\
			\hline
			
			\textbf{real-sim}&0.01&$(.002, .813), 22.376$&$(.001, .207), 7.939$&$(.000, .209), .791$&$(.000, .268), .833$&$(.000, .270), \bm{.730}$&$(.000, .270), \bm{.730}$\\
			22238/50071&0.1&$(.008, .919), 7.09$&$(.001, .207), .826$&$(.001, .700), .428$&$(.001, .584), .416$&$(.001, .585), \bm{.415}$&$(.001, .585), \bm{.415}$\\
			d:20958&0.5&$(.023, .966), 3.680$&$(.001, .207), .794$&$(.001, .755), .245$&$(.003, .810), .190$&$(.003, .829), \bm{.174}$&$(.003, .827), \bm{.181}$\\
			&1&$(.036, .978), 2.570$&$(.001, .207), .794$&$(.007, .880), .121$&$(.007, .875), .125$&$(.007, .894), \bm{.115}$&$(.006, .891), \bm{.109}$\\
			&5&$(.094, .994), .878$&$(.078, .994), .575$&$(.029, .931), .139$&$(.039, .965), .035$&$(.041, .972), \bm{.028}$&$(.044, .974), \bm{.028}$\\
			&10&$(.133, 0.997), .336$&$(.078, .994), \bm{.007}$&$(.069, .993), .007$&$(.099, .986), .019$&$(.092, .991), .009$&$(.094, .991), .009$\\
			\hline
			\textbf{w8a}&0.05&$(.001, .525), .966$&$(.000, .101), .900$&$(.000, .420), .580$&$(.000, .438), \bm{.562}$&$(.000, .428), .572$&$(.000, .428), .572$\\
			1933/62767&0.1&$(.001, .585), .710$&$(.000, .119), .881$&$(.000, .447), .553$&$(.001, .493), .507$&$(.001, .495), \bm{.505}$&$(.001, .499), \bm{.501}$\\
			d:123&0.5&$(.006, .710), .437$&$(.000, .119), .881$&$(.002, .595), .405$&$(.003, .634), .366$&$(.003, .654), \bm{.347}$&$(.003, .667), \bm{.333}$\\
			&1&$(.011, .749), .341$&$(.014, .696), .715$&$(.006, .642), .358$&$(.006, .695), .305$&$(.006, .702), \bm{.298}$&$(.007, .726), \bm{.274}$\\
			&5&$(.048, .823), .177$&$(.014, .696), .305$&$(.013, .701), .299$&$(.046, .805), .195$&$(.033, .818), .182$&$(.036, .827), \bm{.173}$\\
			&10&$(.049, .823), .177$&$(.014, .696), .305$&$(.013, .701), .299$&$(.053, .814), .186$&$(.042, .833), \bm{.167}$&$(.038, .826), \bm{.174}$\\\hline
	\end{tabular}}
	\caption{[(mean false positive rate, mean true positive rate), NP-score] on real-world datasets by different values of the tolerance $\tau$. In the leftmost column, the number of positive/negative instances and feature dimensions (`d') in each dataset. For each dataset, the best results are shown in bold.}
	\label{tab:real_exp}
\end{table*}

In this section we compare the performance of different models by jointly learning the scoring function and threshold in training phase, i.e. output a classifier. To evaluate a classifier under the maximum tolerance, we use \emph{Neyman-Pearson score} (NP-score) \citep{scott2007performance}. The NP-score is defined by
$\frac{1}{\tau} \max\{fpr, \tau\} - tpr$ where $fpr$ and $tpr$ are false-positive rate and true-positive rate of the classifier respectively, and $\tau$ is the maximum tolerance. This measure punishes classifiers whose false-positive rates exceed $\tau$, and the punishment becomes higher as $\tau \to 0$.

\textbf{Settings}. We use the similar setting for classification as for ranking experiments, i.e., for small scale datasets, 30 times stratified hold-out tests are carried out; for large datasets, we instead run 10 rounds. Comparison baselines include: Cost-Sensitive Logistic Regression (CS-LR) which choose a surrogate function that different from  CS-SVM; Bias-Shifting Support Vector Machine (BS-SVM), which first training a standard SVM and then tuning threshold to meet specified false-positive rate; cost-sensitive SVM (CS-SVM). For complete comparison, we also construct a CS-SVM by our out-of-bootstrap thresholding (CS-SVM-OOB), to eliminate possible performance gains comes from different thresholding method, and focus on the training algorithm itself. For all of comparing methods, the hyper-parameters are selected by 5-fold cross-validation with grid search, aims at minimizing the NP-score, and the search scope is extended when the optimal value is at the boundary.  For our $\tau$-FPL, in the ranking stage the regularization parameter $R$ is selected to minimize (\ref{origin rank plm}) , and then the threshold is chosen to minimize NP-score. We test two variants of our algorithms: $\tau$-FPL and $2\tau$-FPL, which corresponding different choice of $\tau$ in learning scheme. As mentioned previously, enlarge $\tau$ can be seen as a top-down approximation towards the original objective.

\textbf{Results}. The NP-score results are given in Table \ref{tab:real_exp}. First, we note that both our methods can achieve the best performance in most of the tests, compared to various comparing methods. Moreover, it is clear that even using the same method to select the threshold, the performance of cost sensitive method is still limited. Another observation is that both of the three algorithms which using out-of-bootstrap thresholding can efficiently control the false positive rate under the constraint. Moreover, $\tau$-FPLs are more stable than other algorithms, which we believe benefits from the accurate splitting of the positive-negative instances and stable thresholding techniques.

\subsection{Scalability}

\begin{figure}[tb]
	\centering
	\includegraphics[width=0.8\linewidth]{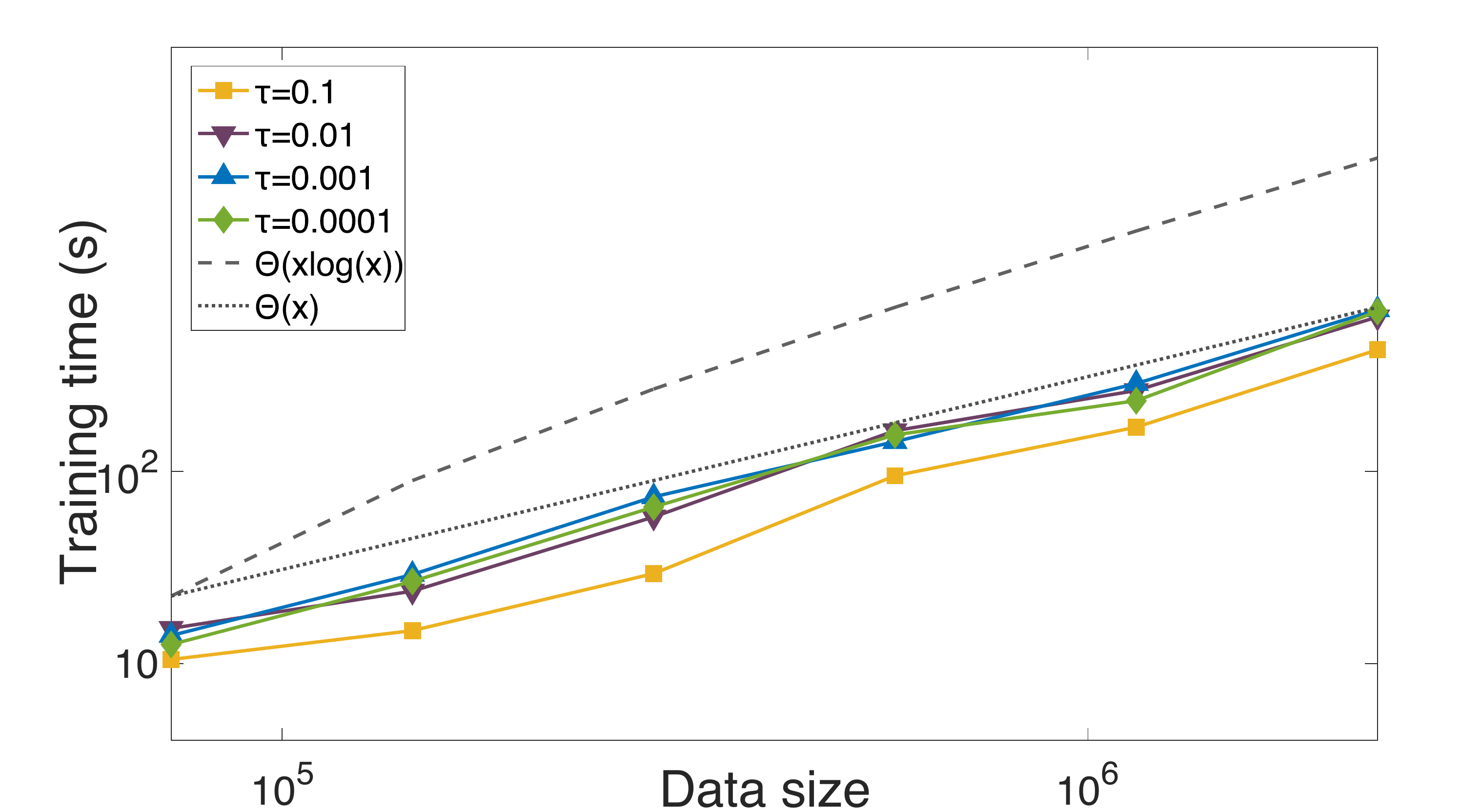}
	\caption{Training time of $\tau$-FPL versus training data size for different $\tau$ (log-log).}
	\label{Fig.scale}
\end{figure}
We study how $\tau$-FPL scales to a different number of training examples by using the largest dataset real-sim. In order to simulate the limit situation, we construct six datasets with different data size, by up-sampling original dataset. The sampling ratio is $\{1,2,2^2,...2^5\}$, thus results in six datasets with data size from 72309 to 2313888. We running $\tau$-FPL ranking algorithm on these datasets with different $\tau$ and optimal $R$ (chosen by cross-validation), and report corresponding training time. Up-sampling technology ensures that, for a fixed $\tau$, all the six datasets share the same optimal regularization parameter $R$. Thus the unique variable can be fixed as data size. Figure \ref{Fig.scale} shows the log-log plot for the training time of $\tau$-FPL versus the size of training data, where different lines correspond to different $\tau$. It is clear that the training time of $\tau$-FPL is indeed linear dependent in the number of training data. This is consistent with our theoretical analysis and also demonstrate the scalability of $\tau$-FPL.

\section{Proofs and Technical Details}

In this section, we give all the detailed proofs missing from the main text, along with ancillary remarks and comments.

\textbf{Notation.} In the following, we define $z_{[i]}$ be the $i$-th largest dimension of a vector $\bm{z} = (z_1,...,z_N) \in \mathbb{R}^N$, define $\alpha_{[0]}^0 = \beta_{[0]}^0 = \infty$, $\alpha_{[n+1]}^0 = \beta_{[n+1]}^0 =-\infty$, define $B_i^j$ as the range $(\beta_{[j]}^0, \beta_{[i]}^0]$ and $B_i$ as the abbreviation of $B_i^{i+1}$.

\subsection{Proof of Theorem \ref{np-hard}: NP-Hardness of $\tau$-$OPT_{L}^{\lambda}$}
Our proof is based on a turing-reduction of the following Maximum Agreement Problem (MAP) to $\tau$-$OPT_{L}^{\lambda}$.
\begin{myplm}
	Define $\mathcal{H'} = \{f(\bm{x}) = \bm{w}^{\top}\bm{x} - b \mid \bm{w} \in \mathbb{R}^p, b \in R\}$. We say $f \in \mathcal{H'}$ and point $(\bm{x}, y) \in R^p \times \{+1, -1\}$ reach an \emph{agreement} if $yf(\bm{x}) > 0$. Given data set $D = \{(\bm{x}_1, y_1),...(\bm{x}_N, y_N) \}$ contains $N$ samples, find a $f \in \mathcal{H'}$ with the maximum number of agreements reached on $D$.
\end{myplm}
It is clear that MAP is in fact equivalent to the problem of binary classification using 0-1 loss and hyperplane hypothesis set. In general, both solving MAP or approximately maximizing agreements to within some constant factor (418/415) are NP-hard \citep{ben2003difficulty}. Now we introduce the decision problem version of MAP (DP-MAP).
\begin{myplm2}
	Given data set $D = \{(\bm{x}_1, y_1),...(\bm{x}_N, y_N) \}$ and any $0 \leq k \leq N$, whether there exists some $f \in \mathcal{H'}$ reaches at least $k$ agreements on $D$ ? If yes, output one of such $f$.
\end{myplm2}
If we have an oracle algorithm $O$ of DP-MAP, we can solve MAP by applying bisection search on $k$ and take the maximum value of $k$ that $O$ output \emph{Yes} as solution. The overall number of calling $O$ is at most $\log(N)$. This naive reduction shows that DP-MAP is also NP-hard. More precisely, it is a NP-Complete problem.

Now we consider to solve DP-MAP by $\tau$-$OPT_{L}^{\lambda}$. This means that, for fixed $0 < \tau < 1$, $\lambda > 0$ and surrogate $L(\cdot)$, we have an oracle algorithm of $\tau$-$OPT_{L}^{\lambda}$, denoted by $O$. We need to design an algorithm so that for all $0 \leq k \leq N$ and data set $D$ with $N$ $p$-dimensional points, we can determine if there exists a hyperplane reaches at least $k$ agreements on $D$ by polynomial calls to $O$. 

\textbf{Case-1: $N - \lfloor \tau N\rfloor < k \leq N$}.~~In this case, we construct the following dataset $D_O$ as input to $O$:
\begin{eqnarray*}
	m &=& 1\\
	n &=& N + A\\
	d &=& p + 1\\
	\bm{x}^+ &=& \bm{0}^{d}\\
	\bm{x}_i^- &=& (-y_i\bm{x}_i^\top, -y_i),~~i = 1,...,N\\
	\bm{x}_{N + j}^- &=& \bm{0}^d,~~j= 1,...,A
\end{eqnarray*}
Here, $A \in \mathbb{N}$ is the smallest non-negative integer that satisfies
\begin{equation}\label{g(A)}
g(A) \triangleq \lfloor\tau(N + A)\rfloor - A = N - k
\end{equation}
We give some properties about $g(A)$.
\begin{lemma}
	(Properties of $g: \mathbb{N} \rightarrow \mathbb{Z}$)
	\begin{enumerate}[itemindent=1em]
		\item $0 \leq g(A) - g(A + 1) \leq 1$; 
		\item $g(A) \leq 0$ when $A \geq \frac{\tau}{1 - \tau}N$;
		\item For any integer $T \in [0, g(0)]$, there exist $A = \mathcal{O}(N)$ such that $g(A) = T$.
	\end{enumerate}	
\end{lemma}
Both of them are easy to verify, so we omit the details. Combine these properties and the fact $ 0 \leq N - k < g(0)$, we know that there must exist some $A \in [0, \lceil\frac{\tau}{1 - \tau}N\rceil]$ satisfies \eqref{g(A)}, and thus the size of dataset we constructed above is linearly related to $N$ and $p$. Now we introduce the following lemma.
\begin{lemma}\label{keq}
	There exists hyperplane reaches at least $k$ agreements iff the minimum value that $O$ found is less than $L(0)$.
\end{lemma}
\begin{proof}
	On the one hand, if there exists a hyperplane $f_0(\bm{x}) = \bm{w}_0^{\top}\bm{x} + b_0 \in \mathcal{H'}$ reaches at least $k$ agreements on $D$, we know $|\mathcal{T}| \leq N - k$ where $\mathcal{T} = \{t \mid (\bm{w}_0^{\top}, b_0)(y_t\bm{x_t}^{\top},  y_t)^{\top} \leq 0\}$. Define $\bm{w_1} = \frac{\lambda(\bm{w}_0^{\top}, b_0)^{\top}}{||(\bm{w}_0^{\top}, b_0)^{\top}||}$.  Now $\bm{w}_1 \in \mathcal{H}^d$ and at most $N - k$ different values of $t \in \{1,...,N\}$ satisfies 
	\begin{equation}\label{kkk}
	\bm{w}_1^T\bm{x}^+ - \bm{w}_1^T\bm{x}_t^- \leq 0.
	\end{equation}
	Note that for any $j = N + 1,...,N + A$ we have $\bm{w}_1^T\bm{x}^+ - \bm{w}_1^T\bm{x}_j^- = 0$, combine these two observations we have: at most $N - k + A$ different values $t \in \{1,...,N+A\}$ satisfies \eqref{kkk}. Thus,
	\begin{eqnarray*}
		&& \bm{w}_1^T\bm{x}^+ - \bm{w}_1^T\bm{x}_{[N- k + A + 1]}^- > 0\\
		\Rightarrow  && L(\bm{w}_1^T\bm{x}^+ - \bm{w}_1^T\bm{x}_{[N- k + A + 1]}^-) < L(0)\\
		\Rightarrow  && L(\bm{w}_1^T\bm{x}^+ - \bm{w}_1^T\bm{x}_{[\lfloor\tau(N + A)\rfloor + 1]}^-) < L(0)\\
		\Rightarrow  && \min \limits_{\bm{w} \in \mathcal{H}^d} L(\bm{w}^T\bm{x}^+ - \bm{w}^T\bm{x}_{[\lfloor \tau n\rfloor + 1]}^-) < L(0)
	\end{eqnarray*}
	The LHS of the last inequality is indeed the minimum value that $O$ output.
	
	On the other hand, it is obvious that the above proof is easily reversible, this completes the proof of lemma.
\end{proof}
By Lemma \ref{keq}, we can determine if there exists a hyperplane reaches at least $k$ agreements by calling $O$ once. If the output minimum value is less than $L(0)$, the hyperplane that $O$ learned is exactly corresponds to the hyperplane that reaches enough agreements on $D$, otherwise there is no such hyperplane. We thus complete the reduction.

\textbf{Case-2: $0 \leq k \leq N - \lfloor \tau N\rfloor$}.~The dataset we used here as input to $O$ is
\begin{eqnarray*}
	m &=& 1\\
	n &=& N + A\\
	d &=& p + 2\\
	\bm{x}^+ &=& \bm{0}^{d}\\
	\bm{x}_i^- &=& (-y_i\bm{x}_i^\top, -y_i, 0)^\top,~~i = 1,...,N\\
	\bm{x}_{N + j}^- &=& (\bm{0}^{p+1\top}, -1)^\top,~~j= 1,...,A
\end{eqnarray*}
$A \in \mathbb{N}$ is the smallest non-negative integer that satisfies
\begin{equation}\label{h(A)}
h(A) \triangleq \lfloor\tau(N + A)\rfloor = N - k
\end{equation}
\begin{lemma}\label{keqq}
	(Properties of $h: \mathbb{N} \rightarrow \mathbb{N^+}$)
	\begin{enumerate}[itemindent=1em]
		\item $0 \leq h(A + 1) - h(A) \leq 1$; 
		\item $h(A) > N$ when $A \geq \frac{1 - \tau}{\tau}N + \frac{1}{\tau}$;
		\item For any integer $T \in [\lfloor \tau N\rfloor, N]$, there exist $A = \mathcal{O}(N)$ such that $h(A) = T$.
	\end{enumerate}	
\end{lemma}
Combine Lemma \ref{keqq} and the fact $\lceil \tau N \rceil \leq N - k \leq N$ we know that the size of dataset constructed above is linearly related to $N$ and $p$. Now the claim in Lemma \ref{keq} is also true in this case, we give a proof sketch below.
\begin{proof}
	We follow the same definitions of $f_0$ and $\mathcal{T}$ as in the proof of Case-1. Define $\bm{w_1} = \frac{\lambda(\bm{w}_0^{\top}, b_0, 1)^{\top}}{||(\bm{w}_0^{\top}, b_0, 1)^{\top}||}$.  Now $\bm{w}_1 \in \mathcal{H}^d$ and we have
	\begin{eqnarray*}
		\bm{w}_1^T\bm{x}^+ - \bm{w}_1^T\bm{x}_i^- &=& (\bm{w}_0^{\top}, b_0)(y_i\bm{x_i}^{\top},  y_i)^{\top},~~~\forall i = 1,...,N\\
		\bm{w}_1^T\bm{x}^+ - \bm{w}_1^T\bm{x}_j^- &=& 1,~~~\forall j = N+1,...N+A
	\end{eqnarray*}
	Thus, \eqref{kkk} holds for at most $N - k$ different values of $t$ in $\{1,...,N+A\}$. This implies
	\begin{eqnarray*}
		&&\bm{w}_1^T\bm{x}^+ - \bm{w}_1^T\bm{x}_{[N- k + 1]}^- > 0\\
		\Rightarrow  && L(\bm{w}_1^T\bm{x}^+ - \bm{w}_1^T\bm{x}_{[N- k + 1]}^-) < L(0)\\
		\Rightarrow  && L(\bm{w}_1^T\bm{x}^+ - \bm{w}_1^T\bm{x}_{[\lfloor\tau(N + A)\rfloor + 1]}^-) < L(0)\\
		\Rightarrow  && \min \limits_{\bm{w} \in \mathcal{H}^d} L(\bm{w}^T\bm{x}^+ - \bm{w}^T\bm{x}_{[\lfloor \tau n\rfloor + 1] }^-) < L(0)
	\end{eqnarray*}
	Above proof can be easily reversed for another direction, we omit the details.
\end{proof}
Combine both Case-1 and Case-2, we complete the reduction from DP-MAP to $\tau$-$OPT_{L}^{\lambda}$. Since DP-MAP is NP-Complete, we conclude that $\tau$-$OPT_{L}^{\lambda}$ must be NP-hard.\\
\textbf{Remark 1}.~It is clear that above reduction can not be used for $\tau = 0$. Indeed, for convex surrogate $L(\cdot)$ the objective of $0$-$OPT_{L}^{\lambda}$ is convex and its global minimum can be efficiently obtained, see \citep{li2014top} for more details.\\
\textbf{Remark 2}.~In problem definition and above proof we implicitly suppose that the minimum of $\tau$-$OPT_{L}^{\lambda}$ can be attained. In general cases, by considering the decision problem of $\tau$-$OPT_{L}^{\lambda}$ one can complete a very similar reduction and we omit details for simplicity here. 
\subsection{Proof of Proposition \ref{prop:upper}}
$\bar{R}_0 \leq \bar{R}_1$ can be easily obtained by combining the relationship between mean and maximum value and the definition that $L(\cdot)$ is an upper bound of $\mathbb{I}(\cdot)$. We now prove $\bar{R}_1 \leq \bar{R}_2$. Define $k = \lfloor\tau n \rfloor + 1$, we have
\begin{eqnarray*}
	&&kL(0) \geq k > \tau n \\
	&\geq& \sum_{j=1}^n L(b - f(\bm{x}_j^-))~\text{(Constraint on FPR)}\\
	&\geq& \sum_{j=1}^k L(b - f(\bm{x}_{[j]}^-))~\text{(Nonnegativity of } L\text{)}\\
	&\geq& kL(\frac{1}{k}\sum_{j=1}^k (b - f(\bm{x}_{[j]}^-)))~\text{(Jensen's Inequality)}\\ 
	&=& kL(b - \frac{1}{k}\sum_{j=1}^k f(\bm{x}_{[j]}^-))\\
	&\Leftrightarrow& b > \frac{1}{k}\sum_{j=1}^k f(\bm{x}_{[j]}^-)~\text{(Monotonicity of } L\text{)}
\end{eqnarray*}
Thus 
\begin{eqnarray*}
	\bar{R}_2 & = &\min_{b \in \mathbb{R}} \frac{1}{m} \sum_{i=1}^m L(f(\bm{x}_i^+) - b)\\
	&\geq& \frac{1}{m} \sum \nolimits_{i=1}^m L(f(\bm{x}_i^+) - \frac{1}{\lfloor \tau n \rfloor + 1}\sum_{j=1}^{\lfloor \tau n \rfloor + 1} f(\bm{x}_{[j]}^-))\\
	&=& \bar{R}_1.
\end{eqnarray*}

\subsection{Proof of Theorem \ref{dual form}}
Since truncated quadratic loss is non-increasing and differentiable, it can be rewritten in its convex conjugate form, that is
\begin{equation}
l(z) = \max \limits_{\alpha \leq 0}\  \{ \alpha z - l_*(\alpha) \}
\end{equation}
where $l_*(\alpha)$ is the convex conjugate of $l$. Based on this, we can rewrite the problem (\ref{learning problem}) as
\begin{eqnarray}\label{plm:1}
\nonumber \min \limits_{\bm{w}} \max_{\alpha \leq 0} \sum_{i=1}^m \alpha_i(\bm{w}^{\top}\bm{x}_i^+ - \frac{1}{k}\sum_{j=1}^k \bm{w}^{\top}\bm{x}_{[j]}^-) - \sum_{i=1}^m l_*(\alpha_i) + \frac{mR}{2}||\bm{w}||^2
\end{eqnarray}
where $\bm{\alpha} = (\alpha_1,..., \alpha_m)^{\top}$ is the dual variable.\\
On the other hand, it is easy to verify that, for $\bm{t} = (t_1,..., t_n)^{\top} \in \mathbb{R}^n$,
\begin{equation*}
\sum_{i=1}^k t_{[k]} = \max \limits_{\bm{p} \in \Omega}\  \bm{p}^{\top}\bm{t}
\end{equation*}
with $\Omega = \{\bm{p} \mid \bm{0} \leq \bm{p} \leq \bm{1}, \bm{1_n}^{\top}\bm{p} = k\}$. By substituting this into (\ref{plm:1}), the problem becomes
\begin{eqnarray*}
	\nonumber \min \limits_{\bm{w}} \max \limits_{\bm{\alpha} \leq \bm{0}^m, \bm{p} \in \Omega} \sum_{i=1}^m \alpha_i(\bm{w}^{\top}\bm{x}_i^+ - \frac{1}{k}\sum_{j=1}^n p_j\bm{w}^{\top}\bm{x}_j^-) - 
	\sum_{i=1}^m l_*(\alpha_i) + \frac{mR}{2}||\bm{w}||^2
\end{eqnarray*}
Now, define $\beta_j = \frac{1}{k} p_j\sum_{i=1}^m \alpha_i$, above problem becomes
\begin{eqnarray*}
	\nonumber \min \limits_{\bm{w}} \max \limits_{\bm{\alpha} \leq \bm{0}^m, \bm{\beta} \leq \bm{0}^n} && \sum_{i=1}^m \alpha_i\bm{w}^{\top}\bm{x}_i^+ - \sum_{j=1}^n \beta_j \bm{w}^{\top}\bm{x}_j^- - \sum_{i=1}^m l_*(\alpha_i) + \frac{mR}{2}||\bm{w}||^2\\	
	s.t. ~&&\sum_{i=1}^m \alpha_i = \sum_{j=1}^n \beta_j, \beta_j \geq \frac{1}{k}\sum_{i=1}^m \alpha_i
\end{eqnarray*}
Notice that this replacement is able to keep the two problems equivalent.\\
Since the objective above is convex in $\bm{w}$, and jointly concave in $\bm{\alpha}$ and $\bm{\beta}$, also its feasible domain is convex; hence it satisfies the strong max-min property \citep{boyd2004convex}, the min and max can be swapped. After swapping, we first consider the inner minimization subproblem over $\bm{w}$, that is
\begin{equation*}
\min \limits_{\bm{w}} \sum_{i=1}^m \alpha_i \bm{w}^{\top}\bm{x}_i^+ - \sum_{j=1}^n\beta_j \bm{w}^{\top}\bm{x}_j^- + \frac{mR}{2}||\bm{w}||^2
\end{equation*}
Here we omit items which does not depend on $\bm{w}$. This is an unconstrained quadratic programming problem, whose solution is $\bm{w}^* = -\frac{1}{mR}(\bm{\alpha}^{\top}X^+ - \bm{\beta}^{\top}\bm{X}^-)^{\top}$, and the minimum value is given as
\begin{equation}
-\frac{1}{2mR}||\bm{\alpha}^{\top}\bm{X}^+ - \bm{\beta}^{\top}\bm{X}^-||^{\top}
\end{equation}
Then, we consider the maximization over $\bm{\alpha}$ and $\bm{\beta}$. By replacing them with $-\bm{\alpha}$ and $-\bm{\beta}$, we can obtain the conclusion of Theorem \ref{dual form}. \qed

\subsection{Proof of Theorem \ref{sol of proj}}
Let $\lambda, \mu, u_i \geq 0, v_i \geq 0, \omega_i \geq 0$ be dual variables, and let $C = \sum_{i=1}^m \alpha_i^0 = \sum_{j=1}^n \beta_j^0$. Then the Lagrangian function of (\ref{proj plm}) can be written as
\begin{eqnarray*}
	\mathcal{L} &=& \frac{1}{2}||\bm{\alpha} - \bm{\alpha_0}||^2 + \frac{1}{2}||\bm{\beta} - \bm{\beta_0}||^2 + \lambda(\sum_{i=1}^m \alpha_i^0 - C) + \\
	& & \mu(\sum_{j=1}^n \beta_j^0 - C) - \sum_{i=1}^m u_i\alpha_i^0 - \sum_{j=1}^n v_i\beta_j^0 + \sum_{j=1}^n(\beta_j^0 - \frac{C}{k}).
\end{eqnarray*}
The KKT conditions are
\begin{eqnarray}
\frac{\partial \mathcal{L}}{\partial \alpha_i^*} &=& \alpha_i^* - \alpha_i^0 + \lambda - u_i = 0 \label{alpha1}\\
\frac{\partial \mathcal{L}}{\partial \beta_j^*} &=& \beta_j^* - \beta_j^0 + \mu - v_j + \omega_j = 0 \label{eq:beta1} \\
\frac{\partial \mathcal{L}}{\partial C} &=& -\lambda - \mu - \frac{1}{k}\sum_{j=1}^n \omega_j = 0 \label{eq:3var}\\
0 &=& u_i\alpha_i^* \label{eq:alpha2}\\
0 &=& v_j\beta_j^* \label{eq:beta2}\\
0 &=& \omega_j(\beta_j^* - \frac{C}{k})\label{eq:beta3} \\
C &=& \sum_{i=1}^m \alpha_i^* \label{eq:alphaC}\\
C&=& \sum_{j=1}^n \beta_j^*\label{eq:betaC}\\
\bm{0} &\leq& \bm{\alpha}^*, \bm{\beta}^*, \bm{u}, \bm{v}, \bm{\omega}
\end{eqnarray}
Consider $\beta_i^*$. By (\ref{eq:beta1}) and (\ref{eq:beta2}),  we know that if $v_j = 0$, then $\beta_j^* = \beta_j^0 - \mu - \omega_j \geq 0$; else if $v_j > 0$, then $0 = \beta_j^* > \beta_j^0 - \mu - \omega_j$. This implies that $\beta_j^* = [\beta_j^0 - \mu - \omega_j]_+$. Following similar analysis we can show that $\alpha_i^* = [\alpha_i^0 - \lambda]_+$.

Further, by (\ref{eq:beta3}), we have that if $\omega_i = 0$, then $\beta_j^* = [\beta_j^0 - \mu]_+ \leq \frac{C}{k}$; else if $\omega_i >  0$, then $\frac{C}{k} = \beta_j^* < [\beta_j^0 - \mu]_+$. Thus $\beta_j^* = \min\{[\beta_j^0 - \mu]_+, \frac{C}{k}\}$. \\\\
Substituting the expression of $\beta_j^*$ and $\alpha_i^*$ into (\ref{eq:alphaC}) and (\ref{eq:betaC}), we have that both constraints (\ref{def C}), (\ref{def mu}) and the closed form solution of $\bm{\alpha}$ and $\bm{\beta}$ holds. Now we verify (\ref{3 var}).

\textbf{Case 1.} First consider $C >0$, we have that if $\omega_j = 0$, then $\frac{C}{k} \geq \beta_j^* = \beta_j^0 - \mu$; else if $\omega > 0$, then $0 < \frac{C}{k} = \beta_j^* = [\beta_j^0 - \mu - \omega_j]_+$ and thus $\beta_j^0 - \mu - \omega_j = \frac{C}{k}$. To sum up we know that $\omega_j = [\beta_j^0 - \mu - \frac{C}{k}]_+$, and thus by (\ref{eq:3var}), (\ref{3 var}) holds.

\textbf{Case 2.} Now Suppose $C = 0$. If this is the case, according to (\ref{eq:alphaC}) and (\ref{eq:betaC}), we have $\bm{\alpha}^* = \bm{\beta}^* = \bm{0}$. By a trivial discussion, above KKT conditions can be reduced to:
\begin{eqnarray}
\lambda &\geq& \alpha_{[1]}^0\\
0 &\geq& \lambda + \mu + \frac{1}{k} \sum_{j=1}^n [\beta_j^0 - \mu]_+ \label{simp KKT}
\end{eqnarray}
Notice that there is no any upper bounded constraint on $\lambda$. Thus, if $\lambda$ and $\mu$ satisfy the simplified KKT condition, by choosing a large enough $\lambda' \geq \lambda$, both (\ref{simp KKT}) and (\ref{3 var}) hold, and optimal solution is still zero vector. This completes the proof of necessity.

At last, notice that KKT condition is the necessary and sufficient conditions in the case of convex problem, and the above transformations are reversible, we complete the proof of sufficiency. \qed

\subsection{Proof of Lemma \ref{c=0}}
First suppose $(\bm{\alpha}^*, \bm{\beta}^*) = (\bm{0}^m, \bm{0}^n)$. Denote corresponding dual variables by $C^*, \lambda^*$ and $\mu^*$. First,  we have
\begin{equation}\label{cond of lambda}
\bm{\alpha}^* = \bm{0}^m \Leftrightarrow \lambda^* \geq \alpha_{[1]}^0
\end{equation}
Moreover, $(\bm{\alpha}^*, \bm{\beta}^*) = (\bm{0}^m, \bm{0}^n)$ implies $C^*=0$, and equality (\ref{def mu}) is automatically holds for arbitrary $\mu \in \mathbb{R}$. Thus the unique constraint on $\mu^*$ is (\ref{3 var}), i.e.
\begin{equation}\label{111}
k\lambda^* + k\mu^* + \sum_{j =1}^n [\beta_j^0 - \mu^*]_+ = 0.
\end{equation}
Consider $f(\mu) = k\mu + \sum_{j =1}^n [\beta_j^0 - \mu]_+$. It is easy to verify that $f(\mu)$ is continuous and piecewise-linear on $\mathbb{R}$. Moreover, if $\mu \in B_t$, we can write $f(\mu)$ as
\begin{eqnarray}
f(\mu) = \sum_{i=1}^t \beta_{[i]}^0 + (k - t)\mu
\end{eqnarray}
Thus, $f(\mu)$ is strictly decreasing in $B_{k+1}^{n+1}$, strictly increasing in $B_0^k$, and the minimum is achieved in $B_k$, which is $\sum_{i=1}^k \beta_{[i]}^0$. Combine this conclusion with (\ref{cond of lambda}) and (\ref{111}), we have
\begin{eqnarray*}\label{abc}
	k\alpha_{[1]}^0 + \sum_{j=1}^k \beta_{[i]}^0 &=& \min \limits_{\lambda^*, \mu} \{ k\lambda^* + k\mu + \sum_{j =1}^n [\beta_j^0 - \mu]_+\} \\
	&\leq& k\lambda^* + k\mu^* + \sum_{j =1}^n [\beta_j^0 - \mu^*]_+ \\
	&=& 0.
\end{eqnarray*}
This proves the necessity. \\
On the other hand, if $k\alpha_{[1]}^0 + \sum_{j=1}^k \beta_{[j]}^o \leq 0$, by setting $C = 0, \mu = \beta_{[k]}^0$ and $\lambda = -\frac{1}{k} \sum_{j =1}^k \beta_j^0$, one can check that all the optimal conditions in theorem \ref{sol of proj} can be satisfied, and corresponding $(\bm{\alpha}^*, \bm{\beta}^*) = (\bm{0}, \bm{0})$. This completes the proof.$\qed$

\subsection{Redefine $\mu(C)$ }\label{redifine_muc}
As we mentioned in footnote, we need to redefine function $\mu(C)$ to ensure its  one-valued property. We begin by introducing following lemma.

\begin{lemma}\label{sol of mu}
	Consider (\ref{def mu}). Denote $\ C_0 = k(\beta_{[k]}^0 - \beta_{[k+1]}^0)$.
	\begin{enumerate}
		\item If $\ C > C_0$, then there exists a unique $\mu$ satisfying (\ref{def mu}).
		\item If $\ 0 < C \leq C_0$, then arbitrary $\mu \in [\beta_{[k+1]}^0, \beta_{[k]}^0 - \frac{C}{k}]$ can satisfy (\ref{def mu}).
	\end{enumerate}
\end{lemma}

\begin{proof}
For a fixed $C > 0$, consider
\begin{equation}
g(\mu) = \sum_{j =1}^n \min \{[\beta_j^0 - \mu]_+, \frac{C}{k}\} - C
\end{equation}
Obviously it is a continuous and decreasing function, and the range is $[-C$,$\infty)$. By intermediate value theorem, $g=0$ must have a root. Another fact is that $g$ is piecewise-linear and has at most $2n$ breakpoints: $\{\beta_i^0, \beta_i^0 - \frac{C}{k}\}_{i=1}^n$. Thus if the root of $g$ is not unique for some $C$,  all of these roots must be located on a continuous segment with slope 0. Otherwise, if there exists a root only located on segment whose slope is not 0, we can ensure no other root exists.

Let $\mu'$ be a root of $g$, we first show that $\mu' \in (\beta_{[k+1]}^0 - \frac{C}{k}, \beta_{[k]}^0)$. Notice that
\begin{equation*}
\begin{split}
&g(\beta_{[k]}^0) = \sum_{j =k}^{n} 0 + \sum_{j =1}^{k-1} \min \{[\beta_{[j]}^0 - \beta_{[k]}^0]_+, \frac{C}{k}\} - C \leq \sum_{j=1}^{k-1} \frac{C}{k} - C = -\frac{C}{k} < 0,\\
&g(\beta_{[k+1]}^0 - \frac{C}{k}) \geq \sum_{j =1}^{k+1} \min \{[\beta_{[j]}^0 - \beta_{[k+1]}^0 + \frac{C}{k}]_+, \frac{C}{k}\} - C \geq \sum_{j=1}^{k+1} \frac{C}{k} - C = \frac{C}{k} > 0.
\end{split}
\end{equation*}
Thus above claim holds by the intermediate value theorem and the monotonicity of $g$.\\
Suppose $\mu' \in B_i$ and $\mu' + \frac{C}{k} \in B_j$, we have $j \leq k \leq i$, and $g$ can be written as
\begin{equation}\label{aaaa}
0 = g(\mu') = \sum_{t = j+1}^i \beta_{[t]}^0 - (i-j)\mu' - (1-\frac{j}{k})C
\end{equation}
Now we can prove the claims in Lemma \ref{sol of mu}.
\begin{description}
	\item[Case 1.] If $C > C_0$, we know that $\beta_{[k+1]}^0 > \beta_{[k]}^0 - \frac{C}{k}$, thus $B_k$ and $B_k - \frac{C}{k}$ are disjoint. This means $i \neq j$ (once $i=j$, we must get $i=k=j$, and thus there exists $\mu'$ belongs to $B_k$ and $B_k - \frac{C}{k}$ simultaneously, that is a contraction), and by (\ref{aaaa}) we know the slope which $\mu'$ lies on is not 0, thus $\mu'$ is the unique root of $g$.
	\item[Case 2.]If $C \leq C_0$, we know that $\beta_{[k+1]}^0 \leq \beta_{[k]}^0 - \frac{C}{k} < \beta_{[k]}^0$, and thus for all $\mu' \in [\beta_{[k+1]}^0, \beta_{[k]}^0 - \frac{C}{k}]$, $i=k=j$ and (\ref{aaaa}) is a identities. This means that $\mu' = [\beta_{[k+1]}^0, \beta_{[k]}^0 - \frac{C}{k}]$. 
\end{description}
This completes the proof.
\end{proof}
\noindent \textbf{Remark.} Note that in fact in the case of $C > C_0$, we can get a stronger bound on $\mu'$: $\mu' \in (\beta_{[k]}^0 - \frac{C}{k}, \beta_{[k+1]}^0)$ and thus $j < k < i$. This is based on the fact that
\begin{eqnarray*}
	g(\beta_{[k+1]}^0) &=& \sum_{j =k+1}^{n} 0 + (\beta_{[k]}^0 - \beta_{[k+1]}^0) + \sum_{j =1}^{k-1} \min \{[\beta_{[j]}^0 - \beta_{[k]}^0]_+, \frac{C}{k}\} - C 
	\\&<& \frac{C}{k} + \sum_{j=1}^{k-1} \frac{C}{k} - C = 0,\\
	g(\beta_{[k]}^0 - \frac{C}{k}) &\geq& \min \{[\beta_{[k+1]}^0 - \beta_{[k]}^0+\frac{C}{k}]_+, \frac{C}{k}\} + \sum_{j =1}^{k} \min \{[\beta_{[j]}^0 - \beta_{[k]}^0 + \frac{C}{k}]_+, \frac{C}{k}\} - C \\ 
	&>& 0 + \sum_{j=1}^{k} \frac{C}{k} - C = 0. \qed
\end{eqnarray*}

Based on Lemma \ref{sol of mu}, we can redefine $\mu(C)$ as follow.
\begin{eqnarray}
\mu(C)&=&\left\{\begin{array}{ll}
\beta_{[k+1]}^0 & C \in (0, C_0)\\\\
\mu \ satisfies\ (\ref{def mu}) & C \in [C_0, +\infty]
\end{array}\right. \label{fmu}
\end{eqnarray}
This function is one-valued, and both of our discussions below are based on this new formulation.

\subsection{Proof of Lemma \ref{bound}}
When $C^* > 0$, it is obvious that $\lambda^* < \alpha_{[1]}^0$. Now we consider the claim of lower bound.

On the one hand, if $\mu^* + \frac{C^*}{k} > \beta_{[1]}^0$, we have that for all $j \leq n$, $\beta_j^0 \leq \beta_{[1]}^0 < \mu^* + \frac{C^*}{k}$. Thus
\begin{eqnarray*}
	&&0 = f(C^*) = k\lambda^* +k\mu^* + \sum_{j=1}^n[\beta_j^0 - \mu^* - \frac{C^*}{k}]_+ = k(\lambda^* +\mu^*)
\end{eqnarray*}
and then
\begin{eqnarray*}
	0 < C^* = \sum_{j=1}^n \min\{[\beta_j^0 - \mu^*]_+, \frac{C^*}{k}\} = \sum_{j=1}^n [\beta_j^0 + \lambda^*]_+
\end{eqnarray*}
The last equality implies that $\lambda^* > -\beta_{[1]}^0$.

On the other hand, consider the situation of $\mu^* + \frac{C^*}{k} \leq \beta_{[1]}^0$. According to the proof of Lemma \ref{sol of mu}, we know that $\mu^* + \frac{C^*}{k} > \beta_{[k+1]}^0$, thus
\begin{eqnarray*}
	0 = f(C^*) &=& k\lambda^* +k\mu^* + \sum_{j=1}^n[\beta_j^0 - \mu^* - \frac{C^*}{k}]_+\\
	&=& k\lambda^* +k\mu^* + \sum_{j=1}^k[\beta_j^0 - \mu^* - \frac{C^*}{k}]_+\\
	&\leq& k\lambda^* +k\mu^* + k[\beta_{[1]}^0 - \mu^* - \frac{C^*}{k}]_+\\
	&=& k\lambda^* +k\mu^* + k(\beta_{[1]}^0 - \mu^* - \frac{C^*}{k})\\
	&=& k\lambda^* + k\beta_{[1]}^0 - C^*\\
	&<& k\lambda^* + k\beta_{[1]}^0
\end{eqnarray*}
This also means $\lambda^* > -\beta_{[1]}^0$. \qed

\subsection{Proof of Lemma \ref{inc-dec}}
Our proof is based on a detailed analysis of each function's sub-gradient. The correctness of claim \ref{c1} is verified in \citep{liu2009efficient}, so we only focus on claim \ref{c2}, \ref{c3}, \ref{c4}. Due to space limitation, we only detail the key points.

Consider $\mu(C)$. First, according to Lemma \ref{sol of mu}, we know that $\mu(C)$ is well defined in $[0, \infty)$.  Now let us claim that $\mu(C)$ is continuous. It is not difficult to check that $\mu(C)$ is piecewise-linear, and continuous at both of these breakpoints. Thus, $\mu(C)$ is continuous. In order to verify that $\mu(C)$ is decreasing with $C$, we only need to show that the slope of \emph{any} segment of $\mu$ is less than 0.

Like the proof of Lemma \ref{sol of mu}, suppose that $\mu(C) \in B_i$ and $\mu(C) + \frac{C}{k} \in B_j$, we have $j \leq k \leq i$, and obtain a linear relation between $C$ and $\mu$:
\begin{eqnarray*}\label{tttt}
	(i-j)\mu = \sum_{t = j+1}^i \beta_{[t]}^0  - (1-\frac{j}{k})C
\end{eqnarray*}
Thus, if $C > C_0$, according to the remark of Lemma \ref{sol of mu} we know that $i > k > j$. and the slope of $\mu$ is $-\frac{k-j}{k(i-j)} < 0$. Else in the case of $C \leq C_0$, in terms of the definition of $\mu$ we know corresponding slope is 0. In conclusion, we can ensure that $\mu(C)$ is strictly decreasing in $[C_0, +\infty)$, and decreasing in $(0, +\infty)$.

Similar analysis shows that $\delta(C)$ is also piecewise-linear, has at most $O(n)$ breakpoints, and the slope of each segment is $\frac{i - k}{k(i-j)}(i \neq j)$, which is strictly large than 0. In the case of $C \leq C_0$, the slope is $\frac{1}{k} > 0$, leads to the conclusion of  $\delta$ is strictly increasing in $(0, +\infty)$.

At last, consider $f(\lambda)$. Because both $\mu$ and $-\tau$ are decreasing with $\lambda$, it is obviously to see that $f$ is strictly decreasing in $(-\infty, +\infty)$.

Now we prove the convexity of $f(C)$. We prove that both $\lambda(C)$ and $T(C) = k\mu + \sum_{j=1}^n [\beta_j^0 - \delta]_+$ are convex function. The convexity of $\lambda$ is guaranteed in \citep{liu2009efficient}. Now we only discuss the convexity of $T(C)$. If $C \geq C_0$, reuse the definition of $i$ and $j$ above, and define $x = i - k >0$, $y = k - j > 0$, one can verify that, the sub-gradient of $f(C)$ is
\begin{eqnarray*}
	f' = \frac{xy}{k(x+y)} - 1 = \frac{1}{\frac{k}{x} + \frac{k}{y}} - 1 > -1
\end{eqnarray*}
Following the conclusions that $\mu$ is decreasing with $C$, and $\delta = \mu + \frac{C}{k}$ is increasing with $C$, we know that both $x$ and $y$ is increasing with $C$. Thus $f'$ is increasing with $C$, which means that $f$ is convex in $[C_0, +\infty)$.

On the other hand, if $C \leq C_0$, one can easily check that $T(C) = -C$. Thus this moment $f' = -1$, which is larger than the sub-gradient of $f$ in the case of $C > C_0$. Thus $f'$ is increasing in $(0, +\infty)$, and $f$ is convex in $(0, +\infty)$. \qed

\subsection{Proof of Theorem \ref{bound p}}
For the convenience of analysis, we consider the constrained version of the optimization problem (\ref{learning problem}), that is
\begin{equation}
\min \limits_{w \in W} L_{\bar{k}} = \frac{1}{m} \sum_{i=1}^m l(\bm{w}^{\top}\bm{x_i^+} - \frac{1}{k}\sum_{j=1}^k\bm{w}^{\top}\bm{x_{[j]}^-})
\end{equation}
where $W=\{\bm{w} \in \mathbb{R}^d \mid ||\bm{w}|| \leq \rho\}$ is a domain and $\rho > 0$ specifies the size of the domain that plays similar role as the regularization parameter $R$.\\
First, we denote $G$ as the Lipschitz constant of the truncated quadratic loss $l(z)$ on the domain $[-2\rho, 2\rho]$, and define the following two functions based on $l(z)$, i.e.
\begin{eqnarray}
h_l(\bm{x}, \bm{w}) &=& \mathbb{E}_{\bm{x}- \thicksim P-}[l(\bm{w}^{\top}\bm{x} - \bm{w}^{\top}\bm{x}^-)],\\
P_l(\bm{w}, \tau) &=& \mathbb{P}_{\bm{x}^+ \thicksim P^+}(h_l(\bm{x}_i^+ \bm{w}) \geq \tau)
\end{eqnarray}
The lemma below relates the empirical counterpart of $P_l$ with the loss $L_{\bar{k}}$.
\begin{lemma}\label{god lemma}
	With a probability at least $1-e^{-s}$, for any $\bm{w} \in W$, we have
	\begin{equation}
	\frac{1}{m}\sum_{i=1}^m \mathbb{I} (h_l(\bm{x}_i^+,\bm{w}) \geq \delta) \leq L_{\bar{k}},\\
	\end{equation}
	where
	\begin{eqnarray}\label{delta}
	\nonumber \delta = \frac{4G(\rho +1)}{\sqrt{n}} + \frac{5\rho(s+\log(m))}{3n} + 2G\rho\sqrt{\frac{2(s + \log(m))}{n}} + \frac{2G\rho(k-1)}{n}.
	\end{eqnarray}
\end{lemma}
\begin{proof}
	For any $\bm{w} \in W$, we define two instance sets by splitting $\mathcal{S}_+$, that is
	\begin{eqnarray*}
		A(\bm{w}) &=& \{\bm{x}_i^+ \mid \bm{w}^{\top}\bm{x}_i^+ > \frac{1}{k} \sum_{i=1}^k \bm{w}^{\top}\bm{x}_{[i]}^- + 1\}\\
		B(\bm{w}) &=& \{\bm{x}_i^+ \mid \bm{w}^{\top}\bm{x}_i^+ \leq \frac{1}{k} \sum_{i=1}^k \bm{w}^{\top}\bm{x}_{[i]}^- + 1\}
	\end{eqnarray*}
	For $\bm{x}_i^+ \in A(W)$, we define
	\begin{equation*}
	||P -P_n||_W = \mathop{sup} \limits_{||\bm{w}||\leq\rho}|h_l(\bm{x}_i^+, \bm{w}) - \frac{1}{n}\sum_{j=1}^n l(\bm{w}^{\top}\bm{x}_i^+ - \bm{w}^{\top}\bm{x}_j^-)|
	\end{equation*}
	
	Using the Talagrand's inequality and in particular its variant (specifically, Bousquet bound) with improved constants derived in \citep{bousquet2002bennett}, we have, with probability at least $1-e^{-s}$,
	\begin{eqnarray}\label{god ineq1}
	\nonumber ||P -P_n||_W \leq \mathbb{E}||P-P_n||_W + \frac{2s\rho}{3n} + \sqrt{\frac{2s}{n}(2\mathbb{E}||P -P_n||_W + \sigma_P^2(W))}.
	\end{eqnarray}
	We now bound each item on the right hand side of (\ref{god ineq1}). First, we bound $\mathbb{E}||P -P_n||_W$ as
	\begin{eqnarray*}\label{god ineq2}
		\mathbb{E}||P -P_n||_W &=& \frac{2}{n}\mathbb{E}[\mathop{sup} \limits_{||\bm{w}||\leq \rho}\sum_{j=1}^n\sigma_j l(\bm{w}^{\top}(\bm{x}_i^+ - \bm{x}_j^-))]\\
		&\leq& \frac{4G}{n}\mathbb{E}[\mathop{sup} \limits_{||\bm{w}||\leq \rho}\sum_{j=1}^n\sigma_j (\bm{w}^{\top}(\bm{x}_i^+ - \bm{x}_j^-))] \\
		&\leq& \frac{4G\rho}{\sqrt{n}}
	\end{eqnarray*}
	where $\sigma_j$'s are Rademacher random variables, the first inequality utilizes the contraction property of Rademacher complexity, and the last follows from Cauchy-Schwarz inequality and Jensen's inequality. Next, we bound  $\sigma_P^2(W)$, that is,
	\begin{equation*}\label{god ineq3}
	\sigma_P^2(W) = \mathop{sup} \limits_{||\bm{w}||\leq \rho} h_l^2(\bm{x}, \bm{w}) \leq 4G^2\rho^2
	\end{equation*}
	By putting these bounds into (\ref{god ineq1}), we have
	\begin{eqnarray*}
		||P -P_n||_W &\leq& \frac{4G\rho}{\sqrt{n}} + \frac{2s\rho}{3n} + \sqrt{\frac{2s}{n}(4G^2\rho^2 + \frac{8G\rho}{\sqrt{n}})}\\
		&\leq&  \frac{4G(\rho + 1)}{\sqrt{n}} + \frac{5s\rho}{3n} + 2G\rho\sqrt{\frac{2s}{n}}
	\end{eqnarray*}
	Notice that for any $x_i^+ \in A(W)$, there are at most $k-1$ negative instances have higher score than it, we thus have
	\begin{equation*}
	\sum_{j=1}^n l(\bm{w}^{\top}\bm{x}_i^+ - \bm{w}^{\top}\bm{x}_j^-) \leq 2G\rho(k-1)
	\end{equation*}
	Consequently, by the definition of $||P -P_n||_W$ we have, with probability $1 - e^{-s}$,
	\begin{eqnarray*}
		|h_l(\bm{x}_i^+, \bm{w})| \leq ||P -P_n||_W + \frac{1}{n}\sum_{j=1}^n l(\bm{w}^{\top}\bm{x}_i^+ - \bm{w}^{\top}\bm{x}_j^-)\\
		\leq \frac{4G(\rho + 1)}{\sqrt{n}} + \frac{5s\rho}{3n} + 2G\rho\sqrt{\frac{2s}{n}} + 2G\rho\frac{k-1}{n}
	\end{eqnarray*}
	Using the union bound over all $\bm{x}_i^+$'s, we thus have, with probability $1-e^{-s}$,
	\begin{equation}
	\sum_{\bm{x}_i^+ \in A(\bm{w})} \mathbb{I}(h_l(\bm{x}_i^+, \bm{w}) \geq \delta) = 0
	\end{equation}
	where $\delta$ is in (\ref{delta}). Hence, we can complete the proof by $|B(\bm{w})| \leq mL_{\bar{k}}$.
\end{proof}
Based on Lemma \ref{god lemma}, we are at the position to prove Theorem \ref{bound p}. Let $S(W, \epsilon)$ be a proper $\epsilon$-net of $W$ and $N(\rho, \epsilon)$ be the corresponding covering number. According to standard result, we have
\begin{equation}
\log N(\rho, \epsilon) \leq d\log(\frac{9\rho}{\epsilon}).
\end{equation}
By using concentration inequality and union bound over $\bm{w}' \in S(W, \epsilon)$, we have, with probability at least $1-e^{-s}$,
\begin{eqnarray}\label{god god ineq}
\nonumber \mathop{sup} \limits_{\bm{w}' \in S(W, \epsilon)} P_l(\bm{w}', \delta) - \frac{1}{m}\sum_{i=1}^m\mathbb{I}(h_l(\bm{x}_i^+, \bm{w}') \geq \delta)
\leq \sqrt{\frac{2(s+d\log(9\rho/\epsilon))}{m}}
\end{eqnarray}
Let $\bm{d} = \bm{x}^+ - \bm{x}^-$ and $\epsilon = \frac{1}{2\sqrt{m}}$. For $\bm{w}^* \in W$, there exists $\bm{w}' \in S(W, \epsilon)$ such that $||\bm{w}' - \bm{w}^*|| \leq \epsilon$, it holds that
\begin{eqnarray*}
	\mathbb{I}(\bm{w}^{*\top}\bm{d} \leq 0) = \mathbb{I}(\bm{w'}^{\top}\bm{d} \leq (\bm{w}' - \bm{w}^*)^{\top}\bm{d}) 
	\leq \mathbb{I}(\bm{w}'^{\top}\bm{d} \leq \frac{1}{\sqrt{m}}) \leq 2l(\bm{w}'^{\top}\bm{d}).
\end{eqnarray*}
where the last step is based on the fact that $l(.)$ is decreasing and $l(1/\sqrt{m}) \geq 1/2$ if $m \geq 12$.
We thus have $h_b(\bm{x}^+, \bm{w}^*) \leq 2h_l(\bm{x}^+, \bm{w}')$ and therefore $P_b(\bm{x}^*, \delta) \leq P_l(\bm{x}', \delta/2)$.\\
As a consequence, from (\ref{god god ineq}), Lemma \ref{god lemma} and the fact
\begin{eqnarray*}
	L_{\bar{k}}(\bm{w}') \leq L_{\bar{k}}(\bm{w}^*) + \frac{G\rho}{\sqrt{m}}
\end{eqnarray*}
We have, with probability at least $1-2e^{-s}$
\begin{eqnarray*}
	P_b(\bm{w}^*, \delta) &\leq& L_{\bar{k}}(\bm{w}^*) + \frac{G\rho}{\sqrt{m}} + \sqrt{\frac{2s+2d\log(9\rho)+d\log m}{m}},
\end{eqnarray*}
where $\delta$ is as defined in (\ref{delta}), and the conclusion follows by hiding constants.\qed

\section{Conclusion}\label{sec:conclusion}
In this paper, we focus on learning binary classifier under the specified tolerance $\tau$. To this end, we have proposed a novel ranking method which directly optimizes the probability of ranking positive samples above $1-\tau$ percent of negative samples. The ranking optimization is then efficiently solved using projected gradient method with the proposed linear time projection. Moreover, an out-of-bootstrap thresholding is applied to transform the learned ranking model into a classifier with a low false-positive rate. We demonstrate the superiority of our method using both theoretical analysis and extensive experiments on several benchmark datasets.


\acks{This research was mainly done when the first author was an intern at iDST of Alibaba. This work is supported by the National Natural Science Foundation of China (NSFC) (61702186, 61672236, 61602176, 61672231), the NSFC-Zhejiang Joint Fund for the Integration of Industrialization and Information (U1609220), the Key Program of Shanghai Science and Technology Commission (15JC1401700) and the Joint Research Grant Proposal for Overseas Chinese Scholars (61628203). }


\vskip 0.2in
\bibliography{aaai_2018}

\end{document}